\documentclass{article}

\usepackage[preprint,nonatbib]{neurips_2021}

\usepackage[utf8]{inputenc} % allow utf-8 input
\usepackage[T1]{fontenc}    % use 8-bit T1 fonts
\usepackage{hyperref}       % hyperlinks
\usepackage{url}            % simple URL typesetting
\usepackage{booktabs}       % professional-quality tables
\usepackage{amsfonts}       % blackboard math symbols
\usepackage{nicefrac}       % compact symbols for 1/2, etc.
\usepackage{microtype}      % microtypography
\usepackage{amsthm}
\usepackage{ulem}
\usepackage{algorithmic}
\usepackage{mathtools}
\usepackage{eucal}
\usepackage{graphicx}
\usepackage{soul}

\usepackage{xcolor}         % colors

\usepackage[font={small}]{caption}
\usepackage{bm}

\usepackage{footmisc} % enabling footnotes.
\usepackage{booktabs} % for professional tables
\usepackage{url}
\usepackage{comment}
\usepackage{mathtools}
\usepackage{amssymb}
\usepackage{multirow}
\usepackage{amsthm}

\usepackage{caption}
\usepackage{subcaption}
\usepackage{enumitem}

\newcommand{\beq}[1][\vspace{0.3em}]{#1\begin{equation}}
\newcommand{\eeq}{\end{equation}}

\newcommand{\bit}{\vspace{0mm}\begin{itemize}}
\newcommand{\eit}{\vspace{0mm}\end{itemize}}
\newcommand{\ben}{\vspace{0mm}\begin{enumerate}}
\newcommand{\een}{\vspace{0mm}\end{enumerate}}

\newtheorem{prop}{Proposition}

\newtheorem{definition}{Definition}

\newcommand{\LSE}[1]{\underset{#1}{\text{LSE }}}

\usepackage{xcolor}

\usepackage{tikz}
\newcommand{\Plus}{\mathord{\begin{tikzpicture}[baseline=0ex, line width=1, scale=0.1]
\draw (1,0) -- (1,2);
\draw (0,1) -- (2,1);
\end{tikzpicture}}}

\title{Equivariant Neural Network for Factor Graphs}

\author{Fan-Yun Sun\textsuperscript{1},\ \ Jonathan Kuck\textsuperscript{2},\ \ Hao Tang\textsuperscript{3},\ \ \normalfont{and}\ \ \bf{Stefano Ermon}\textsuperscript{1}\\

\textsuperscript{1}Stanford University\ \
\textsuperscript{2}Dexterity\ \
\textsuperscript{3}Cornell University\ \
\\
{\tt\small fanyun@stanford.edu, jonathan@dexterity.ai} \\
{\tt\small haotang@cs.cornell.edu, ermon@stanford.edu}
}

\begin{document}

\maketitle
\begin{abstract}
%\sun{New abstract, please review}
%\tang{Several? can? The current first sentence is somehow less catchy and informative from my perspective.} 
%Permutations on factor graphs change the way that a factor graph is represented but not the underlying probability distribution. 
Several indices used in a factor graph data structure can be permuted without changing the underlying probability distribution. An algorithm that performs inference on a factor graph should ideally be equivariant or invariant to permutations of global indices of nodes, variable orderings within a factor, and variable assignment orderings. However, existing neural network-based inference procedures fail to take advantage of this inductive bias. In this paper, we precisely characterize these isomorphic properties of factor graphs and propose two inference models: Factor-Equivariant Neural Belief Propagation (FE-NBP) and Factor-Equivariant Graph Neural Networks (FE-GNN). FE-NBP is a neural network that generalizes BP and respects each of the above properties of factor graphs while FE-GNN is an expressive GNN model that relaxes an isomorphic property in favor of greater expressivity. Empirically, we demonstrate on both real-world and synthetic datasets, for both marginal inference and MAP inference, that FE-NBP and FE-GNN together cover a range of sample complexity regimes: FE-NBP achieves state-of-the-art performance on small datasets while FE-GNN achieves state-of-the-art performance on large datasets.

%\jonathan{i think we should warm the reader up a bit more as Jacob recommends \url{https://jsteinhardt.wordpress.com/2017/02/28/advice-for-authors/}, here's another stab:}

%Neural networks that are designed to respect the symmetries of input data have better inductive bias than general purpose architectures, at the cost of less general expressivity (e.g. CNNs, GNNs, transformers, etc.).  This trade-off is present when designing neural architectures to perform probabilistic inference; architectures that respect the symmetries of inference queries have the best inductive bias at the cost of less general expressivity.  In this work we precisely characterize the symmetries of probabilistic inference queries, empirically analyze the trade-off between the inductive bias and expressivity of prior neural architectures, and propose two novel neural architectures that improve upon this trade-off by achieving state-of-the-art inference results over a range of data regimes.  Factor-Equivariant Neural Belief Propagation (FE-NBP), a novel architectures that respects a larger set of symmetries than prior work, achieves state-of-the-art performance in the small data regime. Factor-Equivariant Graph Neural Networks (FE-GNN), a more expressive architecture, achieves state-of-the-art performance in the large dataset regime. 

\end{abstract}

\section{Introduction}
%\jonathan{i think it's OK to just talk about factor graphs and probabilistic inference without mentioning the term PGM}
%When is factor graphed used, why is it important
Probabilistic graphical models (PGM) provide a statistical framework for modeling dependencies between random variables. Performing inference on PGMs is a fundamental task with many real-world applications including statistics, physics, and machine learning~\cite{chandler1987introduction,mezard2002analytic,wainwright2008graphical,baxter2016exactly,mcbook}.
Factor graphs are a general way of representing PGMs. 
%It is a structural representation of dependent random variables.
As a data structure, a factor graph can be viewed as a bipartite graph of variables and factors; each factor node 
%in the graph 
indicates the presence of dependencies among the variables it is connected to. Many inference algorithms have been developed to leverage the conditional independence structure imposed by a factor graph representation. Among these, Belief Propagation (BP)~\cite{koller2009probabilistic} has demonstrated empirical success in a variety of applications such as error correction decoding algorithms~\cite{mackay1999good} and combinatorial optimization~\cite{braunstein2004survey}. % is one of the most well-studied. \se{many applications such as ldpc, etc.}
%\se{can also mention inference is hard in the worst case, but BP works well in practice}
%\se{can play up the importance of inference in pgms as a fundamental task with all kinds of applications in stats, ml, physics, etc.} \sun{done.}
%\jonathan{could also mention BP's success on UAI benchmark competition if you have specific rankings that sound good to set the stage that BP is good on UAI}

%\se{there seems to be a missing step in the logic: why use learned inference models. maybe say bp and other inf methods are handcrafted, and cannot adapt to specific instances/distributions. so people have had success with learned..}\sun{done.}
%\jonathan{this connection could be improved.  We want to say BP (and other traditional inference algorithms) are good but can be improved.  Also can highlight that traditional inference algorithms might be good on a particular domain, or might be good on particular domain after tuning hyper-parameters.  We just use data to do well on a domain, and beat the best choice of traditional algorithm}
Despite their empirical success, traditional inference algorithms like BP are handcrafted and perform the same computational procedures for any input PGM, limiting their accuracy. %.\jonathan{can we make this more precise? e.g. "are handcrafted and perform computations with the same structure for any input PGM"} 
Thus, researchers have started to develop \textit{trainable} neural network-based inference models~\cite{kuck2020belief,satorras2021neural,Yoon2018InferenceIP,zhang2019factor}, aiming to enable inference algorithms to adapt by learning from data. % better inference algorithms by making adaptations.% to specific instances.
%As neural networks rise and are proven to be powerful approximators, some~\cite{kuck2020belief,satorras2021neural,zhang2019factor,Yoon2018InferenceIP} have attempted to develop neural network-based inference models. % with belief propagation in the hope of getting the best of both world.
However, prior works have treated factor graphs simply as bipartite graphs,
%of variable and factor nodes, 
overlooking the more complex isomorphisms associated with factor graphs. %factor graphs' other isomorphic properties -- factor graphs are more complex than bipartite graphs as each factor is associated with a multi-dimensional factor potential where each dimension is associated with a variable. \jonathan{"overlooking the more complex isomorphisms associated with factor graphs." maybe move details later on}
%and the order of variable assignments of each variable also affects how factor potentials are represented. 
Recently, \cite{kuck2020belief} presents a description of factor graph isomorphism yet it is incomplete and lacks empirical evidence. In this work, we present a complete description of factor graph isomorphism, with empirical evidence of improved inductive bias.% observed that inference algorithms should be equivariant to permutations of global node orderings and variable orderings within factors. However, they failed to consider another property of factor graphs - permuting variable assignment orderings also results in an isomorphic factor graph. They also did not provide any empirical evidence of their model benefiting from respecting these properties. %Consider a function $F$ that takes a bipartite graph with no node or edge features as input and output prediction on all nodes (e.g., marginal estimates), a common inductive bias of $F$ that has proven to be crucial is for $F$ to be a permutation-equivariant function~\cite{zaheer2017deep}. In other words, the output of $F$ should be equivariant to the ordering of nodes. However, 

We propose two neural network-based inference models that leverage these factor graph isomorphism properties to improve their inductive bias: Factor-Equivariant Neural Belief Propagation and (FE-NBP) and Factor-Equivariant Graph Neural Networks (FE-GNN). FE-NBP is a neural network that adopts the message passing procedure of standard BP but incorporates a neural network module that learns adaptive damping ratios while updating the messages. FE-NBP fully respects factor graph isomorphism. %respects all conditions of factor graph isomorphism including Local Variable Symmetry and 2.  %\tang{If it was me, I would split the following long sentence. e.g., FE-NBP respects all conditions including both cond-1 and cond-2. FE-GNN is a more expressive GNN whle still respecting cond-1. Compared to GNNs, ... Compared to FE-NBP, ...} 
FE-GNN is an end-to-end inference model parameterized by a graph neural network (GNN) tailored for factor graphs. Compared with other existing GNN-based inference models, FE-GNN has a better inductive bias; compared with BP or FE-NBP, its discriminative power can be leveraged when a larger amount of training instances is available.% \se{probably not necessary to be so detailed on Local Variable Symmetry and 2 and 0 etc. something at the level of the abstract will be enough for intro}
% Factor BPNN … respects type 1  type 2 factor graph isomorphism
% Factor GNN …   respects type 1 isomorphism and is highly expressive

Using one of the most common experimental settings for evaluating probabilistic inference algorithms -- marginal inference on Ising models, we show that FE-NBP achieves state-of-the-art performance on small datasets while FE-GNN achieves state-of-the-art performance on large datasets. We further conduct experiments on factor graphs where at least one factor potential is not a symmetric tensor and show dramatic improvement of FE-GNN over other existing GNN-based inference models. This supports our claim that respecting factor graph isomorphism improves the inductive bias of neural architectures that perform inference on factor graphs. We also conduct experiments on real-world UAI-challenge datasets and demonstrate that FE-NBP outperforms existing inference models on MAP inference. 
%\jonathan{let's highlight that we get state of the art results on both marginal and map inference earlier on/in the abstract}

We summarize our contributions as follows:
\begin{itemize}
 \item We identify a previously overlooked isomorphism between factor graphs and empirically demonstrate the effectiveness of incorporating it as an inductive bias.
 
 \item We propose Factor-Equivariant Neural Belief Propagation (FE-NBP), a neural architecture that performs inference on factor graphs. FE-NBP generalizes belief propagation by learning adaptive damping ratios while fully respecting factor graph isomorphism. Experiments conducted on both Ising models and UAI-challenge datasets manifest the effectiveness of FE-NBP.
 
  \item We propose Factor-Equivariant Graph Neural Networks (FE-GNN), an end-to-end GNN-based inference model that respects more conditions of factor graph isomorphisms than existing GNNs. In our experiments, we demonstrate that FE-GNN is superior to other existing GNN-based inference models.% due to its improved inductive bias and that it outperforms all existing models by a significant margin on Ising model datasets with more than 1000 training instances.% can boost performance on inference tasks.
  %\tang{This item seems long and complicated. Consider removing the last sentence? (I am not sure...)}

  \item Empirically, we show that FE-NBP and FE-GNN together cover a wide range of sample complexity regimes and discuss the trade-offs between different classes of models.%: belief propagation, belief propagation neural networks, end-to-end graph neural networks. 
\end{itemize}

%\se{we might need to be carefuly about positioning wrt to jonathan's paper. we won't review it but reviewers might be aware of that paper}\sun{got it. Hopefully the couple of sentences that I added are helpful.}

\begin{figure}
    \centering
    \includegraphics[width=.95\textwidth]{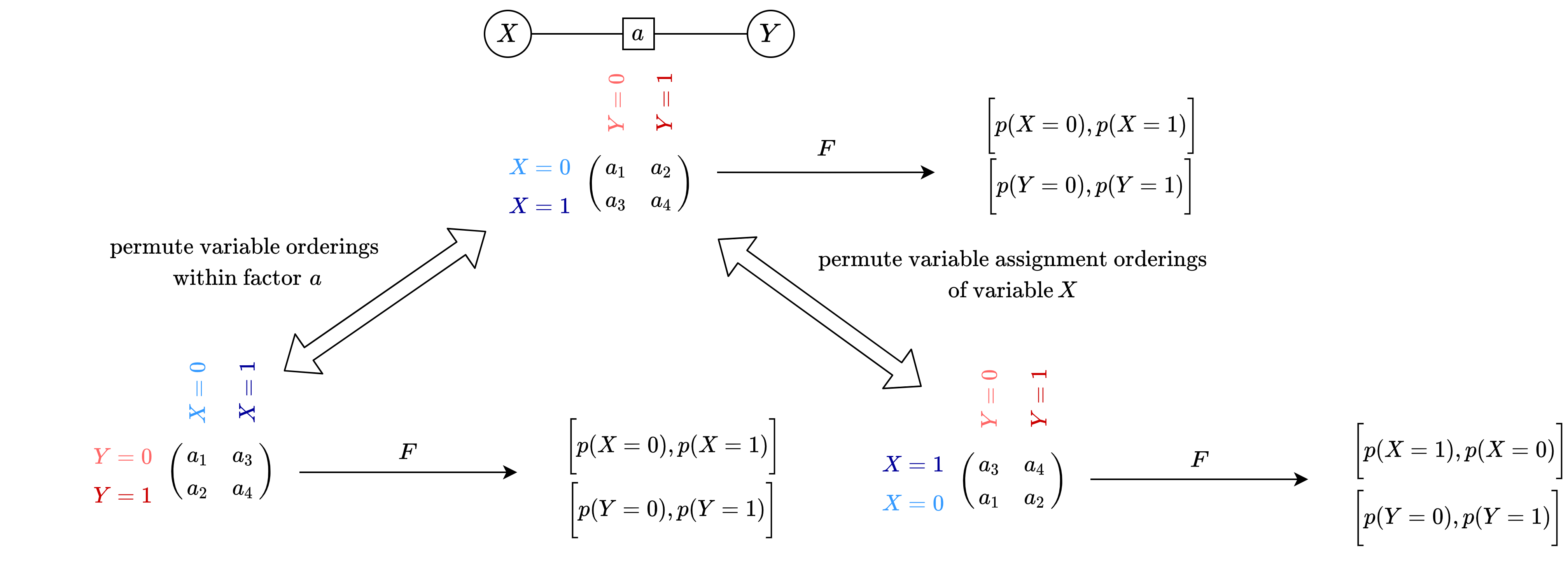}
    \caption{Consider a factor graph with one factor and two variables $X$ and $Y$. Both variables are binary and thus the factor potential $a$ is a 2 by 2 tensor. Note all three factor graphs in the figure are isomorphic. That is, all three represent the same probability distribution and are related by bijective mappings. $F$ is a function or algorithm that takes a factor graph as input and outputs an ordered list of vectors where each vector represents marginal estimates for a variable. %\jonathan{how about making the arrows point between the factor representations.  then replacing the following text with something like, "Arrow 1 depicts the invertible mapping between factors that are isomorphic by Local Variable Symmetry.  Arrow 2 depicts the invertible mapping between factors that are isomorphic by Variable Assignment Symmetry. A function that respects factor graph isomorphism will be equivariant to both mappings."  Should also explain that these per factor mappings can be applied individually to each factor in the factor graph.}
    If we permute variable orderings ($X,Y$ to $Y,X$) within factor $a$ and $F$ gives invariant outputs, $F$ respects Local Variable Symmetry of factor graph isomorphism. If we permute variable assignments within a variable ($X=0,X=1$ to $X=1,X=0$) and $F$ gives equivariant outputs, $F$ respects Variable Assignment Symmetry of factor graph isomorphism (the dimension of marginal vectors permute accordingly)}.%\jonathan{i like the updated figure! however i'm not sure 'permute variable orderings' and 'permute variable assignment orderings' is clear.   can we make this more descriptive?  e.g. 1) transpose the matrix or 2) swap rows/columns.  for tensors this would be 1) permuting the dimensions or 2) permuting all rank-1 tensors associated with a particular dimension.  (btw i think it would be better to always use 'permute')}}
    %\tang{I am not sure, but as a reader, I feel like I need info about what factor graphs are (e.g., how FGs-->probs), especially what potentials are. Also, as one from DL, the word, inference, is also vague for me before this project.} 
    %\tang{Could this figure and caption be more illustrative? This part seems to be critical to show our contributions and concepts? Some minor points that I can think of: (1) The bottom includes two parts (each corresponding to one condition). Splitting them a little bit and inserting some margins would make it clearer? (2) The three inputs of F seem to be a good example of isomorphism factor graphs. How to point out that? (3) The function F and the output of F seems less straightforward and the difference in subscripts is hard to notice.}
    \label{fig:factor_graph_isomorphism}
    \vspace{-5mm}
\end{figure}
\section{Background}
%\TD{Be clear this section is background (by putting background in the title. Have a sentence at the beginning saying that's going on in this section.}
%We describe a general version of belief propagation~\cite{yedidia2005constructing} that operates on factor graphs.

%A factor graph~\cite{kschischang2001factor,yedidia2005constructing} is a general representation of a distribution over $n$ discrete random variables, $\{X_1, X_2, \dots, X_n\}$.  Let $x_i$ denote a possible state of the $i^{th}$ variable.  We use the shorthand $p(\mathbf{x}) = p(X_1 = x_1, \dots, X_n = x_1)$ for the joint probability mass function, where $\mathbf{x} = \{x_1, x_2, \dots, x_n\}$ is a specific realization of all $n$ variables.  Without loss of generality, $p(\mathbf{x})$ can be written as the product

In this section we provide background on factor graphs~\cite{kschischang2001factor,yedidia2005constructing}, belief propagation~\cite{koller2009probabilistic}, and graph neural networks~\cite{gori2005new,scarselli2008graph,kipf2016semi}.%belief propagation neural network~\cite{kuck2020belief} \tang{\sout{BPNN}?}. 
%\tang{Can be removed if necessary.}

\paragraph{Factor Graph}
A factor graph is a compact representation of a discrete probability distribution that takes advantage of (conditional) independencies among variables. 
%tang{Not so important, but the prev sentence seems long and some info seems not so critical.
%For example, to be more concise, we can say "a factor graph is a compact representation of a discrete prob distribution that takes adavantages of ..."
%}
Let $p$ be a distribution defined over $n$ discrete random variables $\{X_1, X_2, \dots, X_n\}$. Let $x_i$ denote a possible assignment of the $i^{th}$ variable.  We use the shorthand $p(\mathbf{x}) = p(X_1 = x_1, \dots, X_n = x_1)$ for the joint probability mass function, where $\mathbf{x} = \{x_1, x_2, \dots, x_n\}$ is a realization of all $n$ variables. Without loss of generality, $p(\mathbf{x})$ can be written as following:%the productLet $p(\mathbf{x})$ be a discrete probability distribution defined over variables $\mathbf{x} = \{x_1, x_2, \dots, x_N\}$ in terms of a factor graph as
%\vspace{-2mm}
\footnotesize
\begin{equation} \label{eq:distribution}
    p(\mathbf{x}) = \frac{1}{Z}  \prod_{a=1}^M f_a(\mathbf{x}_a) ,  \quad \ \ \ \ \ \ \   Z = \sum_{\mathbf{x}} \Bigg ( \prod_{a=1}^M f_a(\mathbf{x}_a) \Bigg ).
%\vspace{-2mm}
\end{equation}
\normalsize
The factor graph is defined in terms of a set of factors $\{f_1, f_2, \dots, f_M\}$, where each factor $f_a$ takes a subset of variable's assignments $\mathbf{x}_a \subset \{x_1, x_2, \dots, x_N\}$ as input and $f_a(\mathbf{x}_a)>0$. $Z$ is the factor graph's normalization constant (or partition function). As a data structure, a factor graph is a bipartite graph with $N$ variables nodes and $M$ factor nodes. Factor nodes and variables nodes are connected if and only if the variable is in the scope of the factor. For readability, we will use $a$ and $b$ to index factors, $i$ and $j$ to index variables, upper-case $X_i$ to indicate variable $i$, and lower-case $x_l$ to indicate a variable assignment.

%\tang{Probably, we can illustrate the concept of factor graphs using Figure 1 as well? saying in the simple factor graph as in Figure 1, factor $a$ determines the joint prob of var $X$ and var $Y$ etc. Moreover, saying this joint prob has multiple equivariant representations which correspond to the three factor graph isomorphism... (I am not sure...)}

\paragraph{Belief Propagation}
Belief propagation (BP) is a method for estimating marginal distribution of variables and the partition function (marginal inference). BP performs iterative message passing among neighboring variable and factor nodes. %\jonathan{Stefano's advice for the last BPNN neurips paper was to present BP in the main paper in it's simplest form for readability and move more complicated details to the appendix \url{https://proceedings.neurips.cc/paper/2020/file/07217414eb3fbe24d4e5b6cafb91ca18-Paper.pdf}}
For numerical stability, belief propagation is generally performed in log-space, and messages are normalized at every iteration.%It is also standard to add a \emph{damping} parameter, $\alpha \in [0,1)$, to improve convergence by taking partial update steps.
Variable to factor messages (variable $X_i$ to factor $f_a$), $\bm{m_{i \rightarrow a}^{(k)}}$, and factor to variable messages (factor $f_a$ to variable $X_i$), $\bm{m_{a \rightarrow i}^{(k)}}$, are computed at every iteration $k$ as
\begin{align} 
     %m_{i \rightarrow a}^{(k)}(x_l) = \tilde{m}_{i \rightarrow a}^{(k)}(x_l) + \alpha  \big(m_{i \rightarrow a}^{(k-1)}(x_l) - \tilde{m}_{i \rightarrow a}^{(k)}(x_l)\big),\,\, \tilde{m}_{i \rightarrow a}^{(k)}(x_l) = -z_{i \rightarrow a} + \sum_{c \in \mathcal{N}(i) \setminus a} m_{c \rightarrow i}^{(k-1)}(x_l),
      m_{i \rightarrow a}^{(k)}(x_l) &= -z_{i \rightarrow a} + \sum_{c \in \mathcal{N}(i) \setminus a} m_{c \rightarrow i}^{(k-1)}(x_l), \label{eq:varToFacMsgs_log} \\
      \tilde{m}_{a \rightarrow i}^{(k)}(x_l) &= -z_{a \rightarrow i} +  \LSE{\mathbf{x}_a \setminus X_i=x_l} \bigg(\Psi_a(\mathbf{x}_a) + \sum_{j \in \mathcal{N}(a) \setminus i} m_{j \rightarrow a}^{(k)}(x_j) \bigg). \label{eq:facToVarMsgs_log}
\end{align}
%\begin{equation}
%\begin{split}
    %m_{a \rightarrow i}^{(k)}(x_l) = \tilde{m}_{a \rightarrow i}^{(k)}(x_l) +  \alpha \big(m_{a \rightarrow i}^{(k-1)}(x_l) - & \tilde{m}_{a \rightarrow i}^{(k)}(x_l) \big), \\ \tilde{m}_{a \rightarrow i}^{(k)}(x_l)  & = -z_{a \rightarrow i} +  \LSE{\mathbf{x}_a \setminus X_i=x_l} \bigg(\Psi_a(\mathbf{x}_a) + \sum_{j \in \mathcal{N}(a) \setminus i} m_{j \rightarrow a}^{(k)}(x_j) \bigg). 
%     \tilde{m}_{a \rightarrow i}^{(k)}(x_l) = -z_{a \rightarrow i} +  \LSE{\mathbf{x}_a \setminus X_i=x_l} \bigg(\Psi_a(\mathbf{x}_a) + \sum_{j \in \mathcal{N}(a) \setminus i} m_{j \rightarrow a}^{(k)}(x_j) \bigg). 
%\end{split}
%\end{equation}
%\se{perhaps try if the align environment results in more readable formatting for the eqs?}
%\se{these eq are not very readable}\sun{which part?}
%\se{the typesetting mostly. would be a lot clearer if we could hide the damping under the rug..}
%\se{need to define neighbors notation?}
%. 
We use the shorthand $m_{i \rightarrow a}^{(k)}(x_l)$ and $m_{a \rightarrow i}^{(k)}(x_l)$ to denote $m_{i \rightarrow a}^{(k)}(X_i = x_l)$ and $m_{a \rightarrow i}^{(k)}(X_i = x_l)$, respectively. Note that $\bm{m_{i \rightarrow a}^{(k)}}$ and $\bm{m_{a \rightarrow i}^{(k)}}$ are vectors of size $|X_i|$ and $m_{i \rightarrow a}^{(k)}(x_l)$ and $m_{a \rightarrow i}^{(k)}(x_l)$ are scalars. $\Psi_a(\mathbf{x}_a) = \ln \left( f_a \left( \mathbf{x}_a \right) \right)$ denotes log factor potentials, $z_{i \rightarrow a}$ and $z_{a \rightarrow i}$ are normalization terms, and $\mathcal{N}(a)$ and $\mathcal{N}(i)$ denote the nodes adjacent to $a$ and $i$ respectively. We use the shorthand $\LSE{}$ for the log-sum-exp function. %: $\LSE{\mathbf{x}_a \setminus X_i=x_l} \Big(\Psi_a(\mathbf{x}_a) \Big) = \ln \left(  \sum_{\mathbf{x}_a \setminus X_i=x_l} \exp \Big(\Psi_a(\mathbf{x}_a)\Big) \right).$ \tang{Must we use $\Psi$ to explain LSE? kind-of confused about the notations. How about $\Phi$?} 
The subscript $\mathbf{x}_a \setminus X_i=x_l$ means that we iterate over all variable realizations of variables connected to factor $a$ except that variable $X_i$ is fixed to an assignment of $x_l$.

%\tang{Comment: I feel like the current background (including factor graph and BP) is either too detailed or not detailed enough. To be more concrete, I feel like, for the professional reader, current background would be long and messy; while for those new to PGM, current background haven't provided them a general high-level review. If it was me, I would make the background more concise and straightforward, e.g., just saying BP performs iterative message passing among factors and variables; then show the formula; at last explain necessary notations. }

The belief propagation algorithm proceeds by iteratively updating variable to factor messages and factor to variable messages until they converge %to fixed values.%, referred to as a fixed point, \tang{Do we need the concept of fixed point now?} 
or a predefined maximum number of iterations is reached. Variable beliefs $b^{(k)}_i$ and factor beliefs $b^{(k)}_a$ are computed to estimate marginals
\begin{equation} \label{eq:calculate_beliefs}
    b^{(k)}_i(x_l) \propto \exp \Big(\sum_{a \in \mathcal{N}(i)} m_{a \rightarrow i}^{(k)}(x_l) \Big),\,\, b^{(k)}_a(\mathbf{x}_a) \propto \exp \bigg(\Psi_a(\mathbf{x}_a) + \sum_{j \in \mathcal{N}(a)} m_{j \rightarrow a}^{(k)}(x_j) \bigg)
\end{equation}

To perform MAP inference with BP, we replace the the log-sum-exp function in Equation \ref{eq:facToVarMsgs_log} with max and decoding variable assignments using variable beliefs $\hat x_i^*=\arg\max_{x_i}b_i(x_i)$.

\paragraph{Graph Neural Networks} Graph Neural Networks (GNNs) take an input graph $G = (V, E)$ with or without node features and edge features and learn representation vectors of all nodes. Modern GNNs follow a neighborhood aggregation strategy~\cite{xu2018representation,gilmer2017neural}, where we iteratively update the representation of a node by aggregating representations of its neighbors. After $k$ iterations of aggregation, a node's representation captures the structural information within its $k$-hop network neighborhood. Formally, the $k$-th layer of a GNN with edge features is
\begin{align} \label{eq:standard_gnn}
h_v^{(k)} = \phi \left( h_v^{(k-1)}, f \left( \left\lbrace \left( h_u^{(k-1)}, e_{uv} \right): u\in \mathcal{N}(v)  \right\rbrace \right) \right)
\end{align}
%\tang{Comment: current formula is beautiful. Just a reminder that this formula cannot represent all GNNs. For example, GAT needs $h_v^{(k-1)}$ in $f$. A more general formula while still somehow beautiful is $h_v^{(k)} = \phi \left( h_v^{(k-1)}, f \left( \left\lbrace \left( h_v^{(k-1)}, h_u^{(k-1)}, e_{uv} \right): u\in \mathcal{N}(v)  \right\rbrace \right) \right)$. The most comprehensive framework of GNNs that I have seen is in "Relational inductive biases, deep learning, and graph networks".}
%\begin{align}
%            %\label{eq:neighbor_agg}
%                \label{eq:combine}      
%    a_v^{(k)} & = \text{AGGREGATE}^{(k)} \left( \left\lbrace h_u^{(k-1)}  : u \in \mathcal{N}(v) \right\rbrace \right), \quad  %\\
%    %    \label{eq:combine}
%    h_v^{(k)}   = \text{COMBINE}^{(k)} \left( h_v^{(k-1)}, a_v^{(k)} \right),
%\end{align}
where $h_v^{(k)}$ is the feature vector of node $v$ at the $k$-th iteration/layer, $\mathcal{N}(v)$ is a set of nodes adjacent to $v$, and $e_{uv}$ is the feature vector of edge $(u, v)$. The choice of $\phi$ and $f$ can be crucial in GNNs \cite{xu2018powerful}.

\section{Methodology}
Ideally, an inference algorithm should produce equivalent (or equivariant) outputs for two isomorphic factor graphs. % \jonathan{Ideally, an inference algorithm should produce equivalent outputs for two isomorphic factor graphs.  (optional: The outputs should be related by the same invertible mapping that transforms between the two input factor graphs.)}% \tang{since they represent the same prob?}. 
This is a property that BP satisfies, but is overlooked by existing works that perform probabilistic inference using neural networks. We aim to address this problem in this paper.%In this paper we develop neural network-based inference models that satisfy this property.%  As we show in section [], this improves the models' inductive bias.

%In this section, we will first define the three conditions of factor graph isomorphism. Then, we introduce Factor-Equivariant Neural Belief Propagation, an inference model that generalizes belief propagation with neural network while respecting all conditions of factor graph isomorphism. At last, we introduce Factor-Equivariant GNN, an end-to-end GNN model that relaxes one condition of factor graph isomorphism in favor of greater expressivity. \se{consider commenting this paragraph, doesn't add much}

% an end-to-end fashion while respecting two out of three conditions of factor graph isomorphism.
% \tang{Comment: I feel like the description of FE-GNN in previous versions is more excited, some words similar to "a more expressive model that relax the constrain".}
%\se{maybe we should reiterate the problem we are trying to solve here. we should say the desired outputs should respect isomorphism, and convince people it is the case}\sun{Done.}
%\se{the more complete definition of isomorphism is a contribution of this paper, so should not be in background. i would put it in this section, and add a discussion of the relationship to jonathan's earlier paper (like in intro, but more in depth now that you have everything set up)}

%\se{i would move this at the beginning of section 3}

\subsection{Factor Graph Isomorphism}
In this section we characterize three conditions of factor graph isomorphism, an equivalence relation between factor graphs.

%\jonathan{It would be nice if we could write down every mapping between factor graphs that represent the same distribution, but this isn't feasible.  (a factor graph with one factor connected to every variable can represent any distribution, so we can't map this to every factor graph.). Could we prove that we write down every equivalence between factor graphs that represent the same distribution and have the same factor-variable connection structure? We would also need to include rescaling factors}

% and summarize it as follows:%
%\se{i'm not a fan of definitions in introduction. i feel the text above is enough to describe what we want}
%\jonathan{how about "we define three isomorphic mappings between factor graphs" or "we define three types of factor graph isomorphism".:}
% \jonathan{how about "we define three mappings between isomorphic factor graphs"}

\begin{definition}[]~  \label{def:factor_graph_isomorphism}
\begin{enumerate}%[start=0]
    \item \emph{Global Symmetry}: permuting the global indices of variables or factors in a factor graph results in an isomorphic factor graph. A function or algorithm $F$ respects Global Symmetry if the output of $F$ is either equivariant or invariant to the permutation of global node ordering. %\jonathan{alternative option: Two factor graphs are isomorphic if their global variable and factor indices are identical up to a permutation. (active vs. passive)}
    \item \emph{Local Variable Symmetry}: permuting the local indices of variables within factors results in an isomorphic factor graph. A function or algorithm $F$ respects Local Variable Symmetry if the output of $F$ is either equivariant or invariant to the permutation of variable orderings within factors. %\jonathan{would it be helpful to describe here with some tensor notation.  e.g. a factor over $k$ variables is a tensor $a_{1, 2, \dots, k}$.  This condition is permuting the dimensions to transform to $a_{\sigma_1(1), \sigma_1(2), \dots, \sigma_1(k)}$ where $\sigma_1$ is a permutation of $\{1,2,\dots,k\}$}
    \item \emph{Variable Assignment Symmetry}: permuting the variable assignments within variables will result in an isomorphic factor graph. A function or algorithm $F$ respects  Variable Assignment Symmetry if the output of $F$ is either equivariant or invariant to the permutation of the orderings of variable assignments within variables. %\jonathan{would it be helpful to describe here with some tensor notation.  e.g. say variable $i$ has cardinality $C$.  This condition is permuting $a_{1, \dots, i, \dots, k}$ to transform to $a_{1, \dots, \sigma_2(i), \dots, k}$ where $\sigma_2$ is a permutation of $\{1,2,\dots,C\}$}
\end{enumerate}
\end{definition}
Find the formal definition of factor graph isomorphism in Appendix A. This concept is further illustrated in Figure~\ref{fig:factor_graph_isomorphism}. Note that whether an algorithm should be either equivariant or invariant to each of the above permutations depend on the output of the inference algorithm\footnote{For an inference algorithm that infers partition function, it should ideally be invariant to each of the above permutations. For an inference algorithm that estimates marginal probabilities of all variables, it should ideally be equivariant to permutations of global indices of nodes and variable assignment orderings and invariant to variable orderings within factors}. % In this paper, we say that ``an inference algorithm is equivariant to all permutations'' for simplicity even though ``depending on the output of the inference algorithm, it is either equivariant or invariant to each of the above permutations'' is more accurate.}. % For the formal mathematical definition of factor graph isomorphism, refer to Appendix A. \jonathan{How about something shorter:
%\jonathan{i like the names, it might be better to use them throughout the paper instead of 'condition 0'.}In this paper, we will refer to \emph{Global Symmetry} as condition 0, \emph{Local Variable Symmetry} as condition 1, and \emph{Variable Assignment Symmetry} as condition 2 of factor graph isomorphism. 
Prior work~\cite{kuck2020belief} identified Global Symmetry and Local Variable Symmetry but failed to identify Variable Assignment Symmetry. As we show in our experiments, respecting Variable Assignment Symmetry results in improved performance.%TThe closest work to consider this concept is \cite{kuck2020belief}: they discussed how inference algorithms should ideally respect the symmetries of input data - permutation of global node orderings and variable orderings within factors (condition 0 and condition 1). However, they overlooked the fact that permutating variable assignment orderings also result in an isomorphic factor graph and inference algorithms should ideally respect this property of factor graphs as well.

\subsection{Factor-Equivariant Neural Belief Propagation}
%\paragraph{Max Belief Propagation Neural Network}
%Inspired by the BPNN framework proposed in \cite{kuck2020belief}, 
Factor-Equivariant Neural Belief Propagation (FE-NBP) is an inference model that incorporates a neural network on top of BP's procedures. In BP, factor-to-variable messages are iteratively updated and our key idea is to learn an adapted damping ratio for those updates:
% Similar to BPNN, the variable-to-factor messages it updated as in Equation 3 but factor-to-variable messages is modified to incorporate a learned operator. %\tang{Another way to describe Factor BPNN: If we consider BP as iterative updates of the factor-to-variable messages (i.e., we only update the f-to-v messages and pass it to the next iteration at each iteration), then we can say `` the new/candidate factor-to-variable messages, $\tilde{m}^{k}_{a\rightarrow i}$, are calculated as Equation 3 and Equation 4. However, we utilize an element-wise adaptive damping mechanism.} 

%More formally, factor-to-variable messages are updated as follows:
%\begin{gather}\label{eq:maxbpnn_damping}
%    m_{a \rightarrow i}^{(k)} = \tilde{m}_{a \rightarrow i}^{(k)} +  \alpha^{(k)}_{a \rightarrow i}\big(m^{(k-1)}_{a \rightarrow i} - \tilde{m}^{(k)}_{a \rightarrow i}\big),\,\, \\[2pt]
%    \alpha^{(k)}_{a \rightarrow i} = \phi_{\text{NN}}(|{m}^{(k)}_{a \rightarrow i}-\tilde{m}^{(k)}_{a \rightarrow i}|, |b_i^{(k)}-\max_{x_a\backslash x_i}b_a^{(k)}|). 
%\end{gather}

%$\lbrack$ HT: I state every details of Factor BPNN here. Please feel free to remove/modify these later on.

%Note on 06 Mar: It seems that the notations are not stable by now. I will change the notations afterwards.

%\paragraph{The overview of Factor BPNN.} We update the factor-to-variable messages using an element-wise adaptive damping mechanism as follows:
\begin{equation} \label{eq:fenbp-1}
    m_{a \rightarrow i}^{(k)}(x_l) = \tilde{m}_{a \rightarrow i}^{(k)}(x_l) +  \alpha^{(k)}_{a \rightarrow i}(x_l) \big(m^{(k-1)}_{a \rightarrow i}(x_l) - \tilde{m}^{(k)}_{a \rightarrow i}(x_l)\big)
\end{equation}
for $a\in [M], i\in [N], l\in [|X_i|]$\footnote{For $K \in \mathbb{N}$, we use $[K]$ to denote $\{1, 2, \ldots, K\}$.} and $k$ denotes the number of current iterations. $\tilde{m}^{(k)}_{a \rightarrow i}(x_l)$ has the same definition as in Equation~\ref{eq:facToVarMsgs_log}. Note that $m_{a \rightarrow i}^{(k)}$ is a vector of size $|X_i|$ and $m_{a \rightarrow i}^{(k)}(x_l)$, being the shorthand of $m_{a \rightarrow i}^{(k)}(X_i = x_l)$, is a scalar in $m_{a \rightarrow i}^{(k)}$.

The damping ratio $\alpha^{(k)}_{a \rightarrow i}(x_l)$ is adaptive and calculated using a neural network module as follows:
\begin{equation} \label{eq:fenbp-2}
    \alpha^{(k)}_{a \rightarrow i}(x_l) = \phi_{\text{NN}}\left({m}^{(k-1)}_{a \rightarrow i}(x_l), \tilde{m}^{(k)}_{a \rightarrow i}(x_l), b_i^{(k)}(x_l), \sum_{\mathbf{x}_a\backslash X_i=x_l}b_a^{(k)}(\mathbf{x}_a), \max_{\mathbf{x}_a\backslash X_i=x_l}b_a^{(k)}(\mathbf{x}_a)\right),
\end{equation}
%where ${m}^{(k-1)(l)}_{a \rightarrow i}$ is the $l$th element of the previous (i.e., at iteration $k-1$) factor-to-variable messages, $\tilde{m}^{(k)(l)}_{a \rightarrow i}$ is the $l$th element of the new factor-to-variable messages calculated the same as the sum-product message passing, 
where $b_i^{(k)}$ and $b_a^{(k)}$ denote variable beliefs and factor beliefs, respectively, following Equation~\ref{eq:calculate_beliefs}. The term $\sum_{\mathbf{x}_a\backslash X_i=x_l}b_a^{(k)}(\mathbf{x}_a)$ denotes factor beliefs summed over all variable realizations of $a$ except that $X_i$ it set to an assignment of $x_l$. $\max_{\mathbf{x}_a\backslash X_i=x_l}b_a^{(k)}(\mathbf{x}_a)$) denotes maximum factor belief achievable when fixing $X_i$ to $x_l$ in $\mathbf{x}_a$. %\tang{To save space, we can compress the prev two sentences by e.g., the term $\max$... is similarly defined, etc.} % to $x_i=l$ which means $$\left(\max_{x_a\backslash x_i}b_a^{(k)}\right)^{(l)}=\max(\{b_a^{(k)}(\mathbf{x}_a)~\forall \mathbf{x} \text{ and } \mathbf{x}_i=l\}).
%In more details, $b_a^{(k)}$ is a multi-dimensional tensor with the same shape as the factor potentials $F_a$. $b_a^{(k)}(\mathbf{x}_a)$ refers to the factor belief when assigning $x_a$ to $\mathbf{x}_a$.
%$\left(\max_{x_a\backslash x_i}b_a^{(k)}\right)^{(l)}$ is then the maximum factor belief when fixing $X_i$ to $x_l$ in $\mathbf{x}_a$% to $x_i=l$ which means $$\left(\max_{x_a\backslash x_i}b_a^{(k)}\right)^{(l)}=\max(\{b_a^{(k)}(\mathbf{x}_a)~\forall \mathbf{x} \text{ and } \mathbf{x}_i=l\}).$$
$\phi_{\text{NN}}$ is a neural network function whose parameters are shared for all updates of the factor-to-variable messages. Output of $\phi_{\text{NN}}$ is a scalar. %, i.e., for all $a\in[M]$, $i\in[N]$, and $l\in [|X_i|]$.
%\paragraph{Implementation Details of Factor BPNN.}
These five input features are designed based on the intuitions of classical improvements of BP, such as incorporating the information of the residual of messages~\cite{gal2006residual} (i.e. $|{m}^{(k-1)}_{a \rightarrow i}(x_l)-\tilde{m}^{(k)}_{a \rightarrow i}(x_l)|$) or inducing more calibrated beliefs~\cite{koller2009probabilistic} (i.e.  $|b_i^{(k)}(x_l)- \sum_{\mathbf{x}_a\backslash X_i=x_l}b_a^{(k)}(\mathbf{x}_a)|$ and $|b_i^{(k)}(x_l)- \max_{\mathbf{x}_a\backslash X_i=x_l}b_a^{(k)}(\mathbf{x}_a)|$) to facilitate faster, more often convergence and better performances.

\begin{prop} \label{prop:fenbp}
For any parameterization of $\phi_{\text{NN}}$, FE-NBP respects Global Symmetry, Local Variable Symmetry, and Variable Assignment Symmetry of factor graph isomorphism in Definiton~\ref{def:factor_graph_isomorphism}. 
\end{prop}
%\tang{Do we need to emphasize that this prop is true for all choices of the parameters of FE-NBP? (I am not sure...)}
%\se{would be good to add these as formal propositions, and add a proof. also should be specific on when this is true. does it hold for any choice of the weights of the nn module?}
We can easily verify that Proposition~\ref{prop:fenbp} hold since messages are updated individually similar to BP. Due to the space limit, we leave more detailed proofs and explanations in Appendix B.

In BPNN~\cite{kuck2020belief}, they also propose to adjust factor-to-variable messages in BP by a neural network. However, FE-NBP offers major advantages over their model: (1) BPNNs does not respect Variable Assignment Symmetry as they apply neural networks directly on message vectors and MLPs are not equivariant to permutation of input vectors. The neural network module $\phi_{\text{NN}}$ in FE-NBP adjusts factor-to-variable messages individually for each variable assignments and thus respects Variable Assignment Symmetry; (2) their neural network only takes in messages as input whereas FE-NBP also incorporates information from variable beliefs and factor beliefs to facilitate the process of calibrating BP's message updates; (3) in their paper, they only conduct experiments on partition function estimation while we conduct experiments with FE-NBP on both marginal inference and MAP inference. For FE-NBP to perform MAP inference, we simply change the log-sum-exp in Equation 3 to log-max-exp and train our model with MAP inference objectives instead of marginal inference objectives. FE-NBP can provide an upper bound and a lower bound for the probability of the output MAP assignment. Find proof in Appendix C.% We can also change the input features to $\phi_{NN}$ accordingly. %In our experiments, we only use the ``sum factor belief'' feature $\left(\sum_{\mathbf{x}_a\backslash X_i=x_l}b_a^{(k)}(\mathbf{x}_a)\right)$ when FE-NBP is used for marginal inference and ``max factor belief'' feature $\left(\max_{\mathbf{x}_a\backslash X_i=x_l}b_a^{(k)}(\mathbf{x}_a)\right)$ when FE-NBP is used to perform MAP inference.

\subsection{Factor-Equivariant Graph Neural Network}
%Recall that a factor graph is a bipartite graph containing factor nodes and edge nodes, with factor potentials associated with every factor node. 
Prior works~\cite{koller2009probabilistic,zhang2019factor} have attempted to use GNN-based models to perform inference. While these GNN-based models respects Global Symmetry, they do not respect further symmetries of fator graphs. % In other words, these proposed models only respect Global Symmetry of Definition~\ref{def:factor_graph_isomorphism}. 
In this section, we first provide a perspective of BP by rewriting its formulation and showing that BP can be considered as a non-parameterized GNN. %\se{will reviewer question whether this is really a new perspective?} \tang{From my perspective, "Rewriting BP as edge/message-based GNNs and showing that it is then easy to design GNNs with the following nice properties" is new. Without "edge/message-based", I am not sure...}
Building on this observation, we propose Factor-Equivariant GNN, a highly expressive GNN architecture that respects Global Symmetry and Local Variable Symmetry.%, respects BP's property of ``no double counting of messages''~\cite{koller2009probabilistic,kuck2020belief}, % \se{was this concept introduced before or does it need some explanation?} \tang{Probably not... Would "add similar to BPNN and cite" be helpful?}, 
%and is highly expressive.%  of Factor-Equivariant Graph Neural Network.% which eventually lead to 

%\tang{A few sentences summarizing the logics of the subsection would be nice? For example, we first provide an new perspective of BP and show that BP can be considered as non-parametrized traditional GNNs over edges/messages. From this perspective, we can easily design FE-GNNs that respect Local Variable Symmetry, avoid double couting, while still being more expressive, etc.}

\paragraph{Rewriting Message Passing of Belief Propagation}
In Equation~\ref{eq:varToFacMsgs_log} and \ref{eq:facToVarMsgs_log}, $m_{i \rightarrow a}^{(k)}(x_l)$ and $m_{a \rightarrow i}^{(k)}(x_l)$ are scalars. In this paragraph, we first rewrite the formulation of a standard log-space BP with damping in terms of vectors and tensors:
%\se{i don't get the difference between (6) and (2)?}
\begin{align} 
     \bm{m_{i \rightarrow a}^{(k)}} & = \bm{\tilde{m}_{i \rightarrow a}^{(k)}} + \alpha  \big(\bm{m_{i \rightarrow a}^{(k-1)}} - \bm{\tilde{m}_{i \rightarrow a}^{(k)}}\big),\,\, \bm{\tilde{m}_{i \rightarrow a}^{(k)}} = -z_{i \rightarrow a} + \sum_{c \in \mathcal{N}(i) \setminus a} \bm{m_{c \rightarrow i}^{(k-1)}}, \label{eq:bp_tensor_form1} \\
    \bm{m_{a \rightarrow i}^{(k)}} & = \bm{\tilde{m}_{a \rightarrow i}^{(k)}} +  \alpha \big(\bm{m_{a \rightarrow i}^{(k-1)}} - \bm{\tilde{m}_{a \rightarrow i}^{(k)}} \big),\,\, \bm{\tilde{m}_{a \rightarrow i}^{(k)}}  = -z_{a \rightarrow i} +  \LSE{X_a \setminus X_i} \bigg(\bm{\Psi_a} + \bigoplus_{j \in \mathcal{N}(a) \setminus i} \bm{m_{j \rightarrow a}^{(k)}}  \bigg). \label{eq:bp_tensor_form2}
\end{align}
%\se{use boldface for vectors maybe to emphasize the difference?}
Vectors and tensors are boldfaced but not scalars. The terms $\bm{m_{i \rightarrow a}^{(k)}}$ and $\bm{m_{a \rightarrow i}^{(k)}}$ are vectors over the states of variable $x_i$ and thus have length $|X_i|$. Log factor potential, $\bm{\Psi_a}$, is a multi-dimensional tensor with shape of $\bigg(\left|X_{a1}\right|, \left|X_{a2}\right|, \cdots, \left|X_{aK}\right| \bigg)$ where $X_{ai}$ denotes the $i$-th variable that associates with factor $a$. We use $X_a$ to denote the set of variables that are connected to factor $a$. Operator $\bigoplus$ denotes ``tensor sum'' (think of outer sum for tensors not just for one-dimensional vectors):$$
\bm{v} = [v_1, v_2, \cdots, v_m], \bm{w} = [w_1, w_2, \cdots, w_n], 
\bm{v} \bigoplus \bm{w} = \begin{bmatrix}
    v_1+w_1 & v_1+w_2 &  \dots  & v_1 +w_n \\
   v_2 + w_1 & v_2+w_2 & \dots  & v_2+w_n \\
    \vdots & \vdots & \ddots &  \vdots\\
    v_m + w_1 & v_m + w_2 & \dots  & v_m+w_n
\end{bmatrix}
$$
Thus, $\bigoplus_{j \in \mathcal{N}(a) \setminus i} \bm{m_{j \rightarrow a}^{(k)}}$ would have a dimension of\newline $\bigg(\left|X_{a1}\right|, \left|X_{a2}\right|, \cdots, \left|X_{a(i-1)}\right|,\left|X_{a(i+1)}\right|, \cdots,  \left|X_{aK}\right| \bigg)$\footnote{We slightly abused the notation by assuming variable $X_i$ is equivalent to the $i$-th variable of factor $a$ ($X_{ai}$)}. It is added with $\Psi_a$ at all corresponding dimensions except at the dimension of $X_i$ and then $\LSE{X_a \setminus X_i}$ would then sum over all other variable dimension except at the dimension of variable $X_i$, resulting a vector of size $|X_i|$.%that variable $X_i$ is fixed with a state of $x_i$.
We make the observation that standard BP is just an instance of GNN without trainable parameters where message passing is performed between message vectors. More specifically, Equation~\ref{eq:bp_tensor_form1} and ~\ref{eq:bp_tensor_form2} can be viewed as instances of Equation~\ref{eq:standard_gnn} ($\bm{m_{i \rightarrow a}^{(k)}}$ and $\bm{m_{a \rightarrow i}^{(k)}}$ corresponds to node feature vectors $h_v^{(k)}$, damping corresponds to function $\phi$, and the aggregation of neighboring messages corresponds to function $f$).% \se{found this description hard to follow}

\paragraph{Factor-equivariant Graph Neural Network}
We propose Factor-Equivariant GNN, a GNN architecture that satisfies two important properties of BP: no double counting of messages and is equivariant to permutation of variable orderings within factors. FE-GNN avoids double counting of messages by performing message passing between message vectors instead of variables and factors. The key insight that makes FE-GNN equivariant to permutation of variable orderings within factors is that instead of treating factor potentials as input node or edge features, we combine factor potentials with message vectors in a similar fashion to BP by incorporating an operation unconventional for GNNs - tensor sum.%with factor potentials with tensor-sum operation $\bigoplus$.

In detail, the model looks like this: % The advantages of this are two folds:  %In another words, Another advantage of performing message passing between message vectors instead of of variables and factors is that it naturally respects the "no double counting" property of BP.
\begin{align} 
  \bm{m_{i \rightarrow a}^{(k)}} &= \phi_{\text{NN}_1}\left(\bm{m_{i \rightarrow a}^{(k-1)}}, \sum_{c \in \mathcal{N}(i) \setminus a} \text{MLP}_1\bigg(\bm{m_{c \rightarrow i}^{(k-1)}} \bigg)\right), \label{eq:fe-gnn1} \\
   % m_{a \rightarrow i}^{(k)} = \text{GRU}_2 \left(m_{a \rightarrow i}^{(k-1)}, \LSE{\mathbf{x}_a \setminus x_i} \bigg( \Psi  \,\,+ \bigoplus_{j \in \mathcal{N}(a) \setminus i} \text{MLP}_2\bigg(m_{j \rightarrow a}^{(k-1)} \mathbin\Vert e_{j \rightarrow a \rightarrow i}\bigg) \right) \\
     \bm{m_{a \rightarrow i}}^{(k)} &= \phi_{\text{NN}_2} \left(\bm{m_{a \rightarrow i}^{(k-1)}}, \LSE{X_a \setminus X_i} \bigg( \bm{\Psi_a}  \,\,+ \bigoplus_{j \in \mathcal{N}(a) \setminus i} \text{MLP}_2\bigg(\bm{m_{j \rightarrow a}^{(k)}} \bigg) \right) \label{eq:fe-gnn2}
\end{align}
%
%Concatentation is denoted by $\mathbin\Vert$. $e_{c \rightarrow i \rightarrow a}$ and $e_{j \rightarrow a \rightarrow i}$ denote the feature of edge ($m_{c \rightarrow i}$, $m_{i \rightarrow a}$) and edge ($m_{j \rightarrow a}$, $m_{a \rightarrow i}$) respectively. Edge features could be hand-crafted features like features used by FE-NBP. After $k$ layers of FE-GNN, variable beliefs (or variable marginals) are estimated by
where $\bm{m_{j \rightarrow a}^{(k)}}$ and $\bm{m_{a \rightarrow i}^{(k)}}$ are node embeddings at layer $k$. $\bm{m_{c \rightarrow i}^{(k-1)}}$ and $\bm{m_{j \rightarrow a}^{(k-1)}}$ are neighboring node embeddings of $\bm{m_{j \rightarrow a}^{(k)}}$ and $\bm{m_{a \rightarrow i}^{(k)}}$ respectively. $\phi_{\text{NN}_1}$ and $\phi_{\text{NN}_2}$ are neural network modules, which are implemented as GRUs~\cite{chung2014empirical} in our experiments. After multiple message passing layers, variable beliefs (or variable marginals) can be estimated by another neural network module such as
$b^{(k)}_i = \text{MLP}_3 \Big(\sum_{a \in \mathcal{N}(i)} \bm{m_{a \rightarrow i}^{(k)}}\Big)$. FE-GNN is trained end-to-end on an objective just like any other GNNs. In our experiments, FE-GNN is trained directly on ground truth marginals of all variables. Additionally, hand-crafted features like those used by FE-NBP can be concatenated with neighboring node embeddings before all neighboring information is aggregated. %More specifcally, we could replace with $m_{c \rightarrow i}^{(k-1)}$ and $m_{j \rightarrow a}^{(k-1)}$ in Equation~\ref{eq:fe-gnn1} with $(m_{c \rightarrow i}^{(k-1)}\mathbin\Vert e_{c \rightarrow i \rightarrow a})$ and $(m_{j \rightarrow a}^{(k-1)} \mathbin\Vert e_{j \rightarrow a \rightarrow i})$; $e_{c \rightarrow i \rightarrow a}$ and $e_{j \rightarrow a \rightarrow i}$ denote the feature of edge ($m_{c \rightarrow i}$, $m_{i \rightarrow a}$) and edge ($m_{j \rightarrow a}$, $m_{a \rightarrow i}$) respectively and $\mathbin\Vert$ is concatenation. 

% should I say that edge features shohuld only be constructed in a certain way if you still want to respect Variable Assignment Symmetry fo factor grfpah isomorphism?
\begin{prop}
For any parameterization of $\phi_{\text{NN}_1}$ and $\phi_{\text{NN}_2}$, FE-GNN respects Global Symmetry and Local Variable Symmetry of factor graph isomorphism in Definiton~\ref{def:factor_graph_isomorphism}. 
\end{prop}

FE-GNN respects Local Variable Symmetry of factor graph isomorphism similar to how BP respects Local Variable Symmetry. The term $\LSE{X_a \setminus X_i}\bigg(... \bigg)$ is not subject to the orderings of variables within factor $a$ since each transformed message vector is summed with factor potentials $\bm{\Psi_a}$ at the dimension that corresponds to the particular variable, and $\LSE{X_a \setminus X_i}$ eventually sums out all the dimension except for the dimension of a particular variable $X_i$.

\begin{comment}
\begin{table}[h!] 
\renewcommand*{\arraystretch}{2.5}
\begin{center}
\scalebox{.6}{
\begin{tabular}{ |c|c|c| } 

\hline
& compute variable to factor messages ($v \rightarrow f$)  &  compute factor to variable messages ($f \rightarrow v$)  \\
\hline

log-space BP &    $m_{i \rightarrow a}^{(k)} = \alpha m_{i \rightarrow a}^{(k-1)} + (1 - \alpha) \Bigg[ -z_{i \rightarrow a} + \sum_{c \in \mathcal{N}(i) \setminus a} m_{c \rightarrow i}^{(k-1)} \Bigg]$ & $ m_{a \rightarrow i}^{(k)} = \alpha m_{a \rightarrow i}^{(k-1)} + (1 - \alpha) \Bigg[ -z_{a \rightarrow i} + \LSE{\mathbf{x}_a \setminus x_i} \bigg( \phi_a(\mathbf{x}_a) \Plus \sum_{j \in \mathcal{N}(a) \setminus i} m_{j \rightarrow a}^{(k)} \bigg) \Bigg]$ \\
\hline
BPNN & & \\
\hline
Factor GNN & $ m_{i \rightarrow a}^{(k)} = \phi \left(m_{i \rightarrow a}^{(k-1)}, f\left(\left\{ m_{c \rightarrow i}^{(k-1)} : c \in \mathcal{N}(i) \setminus a\right\} \right)\right)$ & \\
\hline
\end{tabular} 
}
\end{center} 
\caption{Graph Neural Network equations.} \label{table:gnn}
\end{table}
\end{comment}

\section{Experiments}
We conduct three experiments to support our claims of contribution. First, we conduct experiments on Ising model datasets of different sample sizes to demonstrate the competence of FE-NBP and FE-GNN on marginal inference. Second, we evaluate FE-GNN on a synthetic dataset of ``asymmetric'' factor graphs to support our claim that FE-GNN is superior to other GNN-based models as it respects Local Variable Symmetry. Finally, we conduct experiments on real-world UAI-challenge datasets and show that FE-NBP outperforms BPNN~\cite{kuck2020belief} by a large margin, supporting our claim that respecting Variable Assignment Symmetry is a beneficial inductive bias to incorporate.

\begin{figure}[t]
\centering
    \includegraphics[width=.9\textwidth]{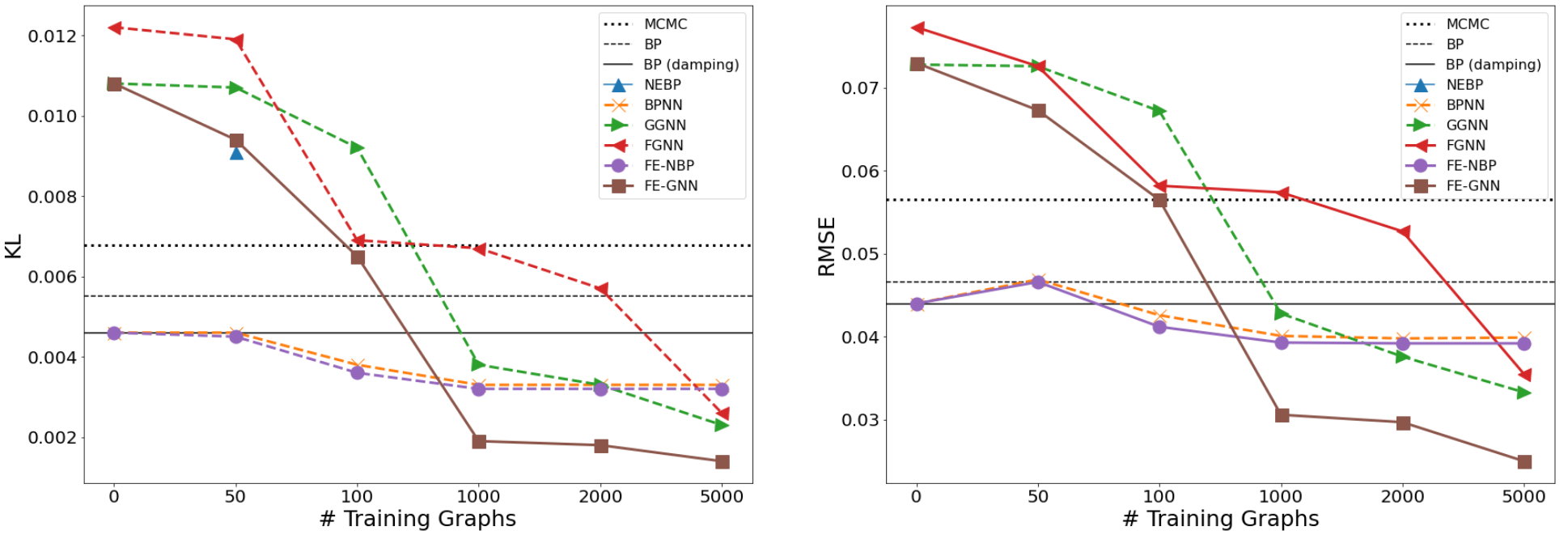}
    \caption{Left: KL divergence between estimated marginals and ground truth for each model trained on Ising model datasets of different sample sizes. Right: RMSE between estimated marginals and ground truth for each model trained on Ising model datasets of different sample sizes.}% \jonathan{is NEBP on the right, i don't see it.  also, might be good to clarify in the main text that you grabbed this number from their paper and that's why there's only one point on the left.}}
    \label{fig:results}
    \vspace{-3mm}
\end{figure}

\subsection{Ising Models}
The goal of this experiment is to compare FE-NBP and FE-GNN with other existing inference models and to investigate how the amount of training data and the expressiveness of a model affect the performance of each inference model.% and discuss which model would be most suitable under what scenario. 
\vspace{-2mm}
\paragraph{Experiment setting}
Ising model is a type of Binary Markov Random Field and is usually represented with $(\mathbf{J}$, $\mathbf{b})$ where $\mathbf{b}$ biases individual variables and $\mathbf{J}$ couples pairs of neighbor variables. Following a common experimental settings to evaluate inference algorithms~\cite{kuck2020belief,satorras2021neural,Yoon2018InferenceIP}, we generate 4x4 grid structured Ising models $(\mathbf{J}$, $\mathbf{b})$ where $\mathbf{x} \in\{+1,-1\}^{|\mathcal{V}|}$, $b_i \sim \mathcal{N}(0, 0.25^2)$, $J_{ij} \sim \mathcal{N}(0, 1)$ and $|\mathcal{V}|=16$ is the number of variables. Ground truth marginals can be computed exactly using junction tree algorithm~\cite{lauritzen1988local}. We vary the amount of data the model is trained on and evaluate them on the same testing set of 1000 factor graphs. % \tang{Although may not important for $4\times 4$ grids, however, as far as I know, junction tree is only efficient on \textbf{attractive} (i.e., $\textbf{J}>0$) Ising models?} 
 %In this experiment, we generate 4x4 grid-structured Ising models $(\mathbf{J}$, $\mathbf{b})$ where $\mathbf{x} \in\{+1,-1\}^{|\mathcal{V}|}$, $b_i \sim \mathcal{N}(0, 0.25^2)$, $J_{ij} \sim \mathcal{N}(0, 1)$ and $|\mathcal{V}|=16$ is the number of variables.

% Ising models are generated in a similar fashion as the last experiment except that Ising models in this experiment have a symmetric $\mathbf{J}$ ($J_{ij} = J_{ji}$). % can be represented as $(\mathbf{J}$, $\mathbf{b}$) where $\mathbf{x} \in\{+1,-1\}^{|\mathcal{V}|}$, $b_i \sim \mathcal{N}(0, 0.25^2)$, $J_{ij} \sim \mathcal{N}(0, 1)$ and $|\mathcal{V}|$ is the number of variables. 
%More specifically, for a factor $a$ that couples variable $i$ and $j$, it has a factor potential of $f_a=\begin{bmatrix}
%    e^{J_{ij}} & e^{-J_{ij}} \\
%   e^{-J_{ij}} & e^{J_{ij}}  \\
%\end{bmatrix}$.% $J_{ij} = J_{ji} = 0$ if variable $i$ and $j$ is not coupled by any factor.

We evaluated FE-NBP and FE-GNN against the following baselines: Markov chain Monte Carlo~\cite{neal1993probabilistic}, Belief Propagation (BP)~\cite{koller2009probabilistic}, BP with damping of 0.5~\cite{Murphy1999LoopyBP}, Neural Enhanced Belief Propagation (NEBP)~\cite{satorras2020neural}, Belief Propagation Neural Networks (BPNN)~\cite{kuck2020belief}, Gated Graph Neural Network (GGNN)~\cite{li2015gated}, and Factor Graph Neural Network (FGNN)~\cite{zhang2019factor}. For more details of these baselines, refer to the Related Works section. Note that the authors of NEBP did not release their code, thus we take the KL divergence results reported on their paper as we have the same experimental settings. For more model and training details, refer to Appendix D.%
\vspace{-1mm}
\paragraph{Results}
%In we present KL divergence and RMSE between estimated marginals and the ground truth for each model trained for datasets with different sample sizes. 
As shown in Figure~\ref{fig:results}, the performance of BP, Damping BP, and MCMC are unchanged as they are not learnable algorithms. % When there are no training data available, Damping BP, BPNN, and FE-NBP have the best performance. 
Observe that FE-NBP and BPNN, when untrained, both correspond to Damping BP as we initialize all neural network parameters to zero (sigmoid function at 0 equals 0.5). When we have 50 or 100 training instances available, FE-NBP performs the best and all the GNN-based models performs poorly. % worse than Damping BP since the number of training graphs is insufficient to successfully train models with high expressivity and complexity. We can observe that 
As we increase the number of training instances to more than 1000, FE-GNN outperforms all baselines and all GNN-based models continue to improve while BPNN and FE-NBP plateaued. The results can be interpreted by the classic bias-variance trade-off: learning algorithms need (inductive) bias tailored to specific learning problems in order to make the target function easier to approximate; learning algorithms also need expressivity so that they can adapt to training data but higher ``flexibility'' can lead to higher variance. BPNN and FE-NBP are algorithms with high (inductive) bias but are relatively less expressive because they are only learning on top of the backbone of BP, whereas GNN-based models are end-to-end parameterized models with less bias but higher expressiveness. In other words, BPNN and FE-NBP are algorithms in the lower sample complexity regimes while GNN-based models are algorithms in the higher sample complexity regimes.

\subsection{Factor-Equivariant GNN on Asymmetric Binary Markov Random Field}
To further show that the importance of incorporating Local Variable Symmetry as an inductive bias for neural network-based inference models, we compare FE-GNN with existing GNN-based inference models on a dataset of \emph{asymmetric} binary markov random field (BMRF)~\cite{merle2019turing}. To clarify, asymmetric means that at least one factor potential is an asymmetric tensor.% In asymmetric MRFs, $J$ is not a symmetric matrix ($J_{ij} \neq J_{ji}$). 

%The Ising model is a loopy graphical model where the performance of Belief Propagation may significantly degrade due to cyclic information. Therefore, in this experiment we include a stronger baseline where Belief Propagation messages are damped to reduce the effect of cyclic information \parencite{koller2009probabilistic}. 

\paragraph{Experiment setting}
In this experiment, we generate 6000 instances of asymmetric BMRF. 5000 instances are used for training and 1000 for testing. BMRF are generated in a similar fashion to the previous experiment except that now a factor $a$ that couples variable $i$ and $j$ would have a factor potential of $f_a=\begin{bmatrix}
    e^{J_{ij}+J_{ji}} & e^{-2*J_{ij}} \\
   e^{-2*J_{ji}} & e^{J_{ij}+J_{ji}}  \\
\end{bmatrix}$ where $J_{ij} \sim \mathcal{N}(0, 1)$ and $J_{ji} \sim \mathcal{N}(0, 1)$. $J_{ij} = J_{ji} = 0$ if variable $i$ and $j$ is not coupled by any factor. As for baselines, we compared FE-GNN with Gated Graph Neural Network (GGNN)~\cite{Yoon2018InferenceIP}, Factor Graph Neural Network (FGNN)~\cite{zhang2019factor}, BP~\cite{koller2009probabilistic}, and BP with damping (0.5)~\cite{Murphy1999LoopyBP}. We follow the same model and training settings as the previous experiment.%, refer to Appendix D.

\begin{table}[t]
\centering
\scalebox{.8}{
\begin{tabular}{lccccc}
	\toprule
Model & BP              & BP (damping) & GGNN   & FG-GNN & FE-GNN \\ \midrule
KL    & 0.0286          & 0.0236       & 0.3770 & 0.3780 & \textbf{0.0109} \\
RMSE  & 0.0818 & 0.0748       & 0.3877 & 0.3877 & \textbf{0.0584} \\ \bottomrule
\end{tabular}
}
\vspace{1mm}
\caption{KL divergence and RMSE between estimated marginals and ground truth for FE-GNN and other baselines on a dataset of asymmetric binary markov random field.}% generalizability on the UAI-challenge datasets.}
\label{tb:asymmetric}
\vspace{-5mm}
\end{table}

\paragraph{Results}
In Table~\ref{tb:asymmetric}, we present KL divergence and RMSE between estimated marginals and ground truth for each model. FE-GNN outperforms all other models by a significant margin. More importantly, observe how other GNN-based inference models perform poorly. The reason is that existing GNN-based inference models like GGNN or FGNN do not respect Local Variable Symmetry of factor graph isomorphism. More specifically, they treat factor potentials as node/edge features and do not take variable orderings within a factor into consideration. As a result, they fail to learn effectively when factor potentials are asymmetric tensors. This experiment shows that FE-GNN, designed with tensor operations that take variable orderings into consideration, is able to learn on more complex factor graphs given the same amount of training data due to its superior inductive bias.

\subsection{Factor-Equivariant Neural BP on the UAI-challenge datasets}
In this experiment, we evaluate FE-NBP on MAP inference with the UAI-challenge datasets\footnote{\url{http://sli.ics.uci.edu/~ihler/uai-data/}}. We did not evaluate GNN-based models on UAI-challenge datasets as the number of instances in those datasets is very limited and is thus not suitable for evaluating highly expressive models (models with high sample complexity). % For each category, we keep the samples with pre-computed  log-scores of the optimal MAP assignment (denoted as MPE in the dataset).}. \tang{Comment: the datasets don't have the MAP assignment, but instead only provide the log-scores of the optimal MAP assignment (i.e., the MPE). The log-scores of a state is $\sum_{a=1}^M\log f_a(\mathbf{x}_a)$}% The properties of the resulting dataset are as follows:% We only keep the categories whose factor graphs are small enough to fit into the GPU memory. For each category, we only keep the samples with pre-computed MPE. The properties of the resulting dataset are as follows:
%\begin{table}[!ht]
%    \centering
%    \caption{Properties of the UAI-challenge dataset.}
%    \small
%   \begin{tabular}{lcccc}
%    \toprule
%         &  Grids & Segment & ObjDetect & DBN \\
%    \midrule 
%         \#samples & 11 & 100 & 116 & 66 \\
%         \#variables (mean) & 290.9 & 229.1 & 60 & 780.2  \\
%         var-cardinality (max) & 2 & 21 & 21 & 2 \\
%    \bottomrule
%    \end{tabular}
%\end{table} 
\paragraph{Experiment setting}
%We adopt the same evaluation metric as the \href{https://www.cs.huji.ac.il/project/PASCAL/eval.php}{UAI-2012} challenge:
%$$ \frac{1}{|\mathcal{D}|}\sum_{G\in\mathcal{D}}\left|\frac{\log\text{score}(x^*)-\log\text{score}(\hat x^*)}{\log\text{score}(x^*)}\right|,$$
%where $G$ stands for the factor graphs, $\mathcal{D}$ is the dataset, $x^*$ is the pre-computed true MAP assignment provided by the dataset, and $\hat x^*$ is the estimated MAP assignment. We compare MaxBPNN with the following baselines: {\color{red} more baselines?} Best-BP: It stands for the best performance that can be achieved using standard BP. In detail, for each sample, we run BP for many times by enumerating all possible choices of hyper-parameters (e.g. damping ratio, update strategy, complex rounding algorithm). The assignment with maximum logscores among all runs of BP will be the best assignment that can be found using standard BP. BP-10: It stands for the untrained version of MaxBPNN-V2. In detail, it has the same hyper-parameter as MaxBPNN except for the damping mechanism. So, BP-10 updates messages in parallel, only iterates for 10 times, and decode the assignments by simply performing argmax over variable beliefs.
%We randomly split the dataset into a training (70\%) and a testing (30\%) dataset. We only tune the initial weights of MaxBPNN (12 choices). The same initial weights are utilized for tuning BP-10.

We randomly split the dataset into a training (70\%) and a testing (30\%) dataset, and adopt the same evaluation metric as the UAI-2012 challenge: $$ \frac{1}{|\mathcal{D}|}\sum_{G\in\mathcal{D}}\left|\frac{\log\text{score}(x^*)-\log\text{score}(\hat x^*)}{\log\text{score}(x^*)}\right|,$$
where the log-score of a state is defined as $\sum_{a=1}^M\log f_a(\mathbf{x}_a)$, $G$ stands for the factor graphs, $\mathcal{D}$ is the dataset, $x^*$ is the true MAP assignment, and $\hat x^*$ is the estimated MAP assignment.

We compare FE-NBP with BPNN~\cite{kuck2020belief} and multiple traditional inference algorithms including BP~\cite{koller2009probabilistic}, BP with damping of 0.5~\cite{Murphy1999LoopyBP} (with standard argmax decoding), BP with an advanced heuristic sequential MAP-assignment decoding mechanism\footnote{implemented by libdai, \url{https://staff.fnwi.uva.nl/j.m.mooij/libDAI/}}~\cite{Mooij_libDAI_10}, MPLP~\cite{globerson2007fixing}, and different variants of the local search algorithms~\cite{pearl1984heuristics, Dept.2015}. 
% To show the improvements of FE-NBP over non-trainable BP, we also incorporate the simplest baseline named as BP-10, i.e., the belief propagation algorithm with the maximum iteration number set as 10 and with the classical $\arg\max$ decoding mechanism. 
For model and training details, refer to Appendix E.

%\subsubsection{Intra-category generalizability}
%\label{sec:intra-uai}
\paragraph{Results} In Table~\ref{tab:intra-uai}, we show the performances of different MAP inference models and algorithms on the UAI-challenge datasets. FE-NBP consistently outperforms BPNN by a significant margin on all four datasets. That supports our claim that respecting Variable Assignment Symmetry of factor graph isomorphism is an effective inductive bias since one of the biggest distinction between FE-NBP and BPNN is that FE-NBP respects Variable Assignment Symmetry but BPNN does not. FE-NBP also outperforms BP and BP (damping) which indicates the benefits of learning and adapting, even on small datasets. Compared with other traditional inference algorithms, FE-NBP outperforms all baselines on the Grids and DBN datasets while achieving competitive performance on the other two datasets. %\tang{Note that the search algorithms perform poorly on the ObjDetct and DBN datasets because there are many assignments with zero probability in their factor graphs.}

\begin{table}[t]
	\centering
	\scalebox{.8}{
	\begin{tabular}{lcccc}
		\toprule
		&  Grids & Segment & ObjDetect & DBN \\
		\midrule
		\#samples & 11 & 100 & 116 & 66 \\
		\#variables (mean) & 290.9 & 229.1 & 60 & 780.2  \\
		var-cardinality (max) & 2 & 21 & 21 & 2 \\
		\midrule
		\midrule
		best-first-search & 0.15 & 0.37 & inf & inf \\
		beam-search & 0.13 & 0.37 & inf & inf \\
% 		random-search & - & - & - & - \\
        MPLP & 0.25 & \textbf{.003} & 0.28 & 0.33 \\
		BP & 0.71 & 4.27 & 0.98 & 0.77 \\
		BP (damping) & 0.22 & 0.09 & 0.62 & 0.18 \\
		%Best-BP (damping) & 0.22 & 0.09 & 0.20 & - \\
 		BP+seq-decoding  & 0.41 & 0.03 & \textbf{0.01} & 0.25 \\
		BP+seq-decoding (damping) & 0.22 & 0.08 & 0.03 & 0.05 \\
		%Best-BP+seq-decoding (damping) & 0.10 & 0.03 & 0.01 & 0.04 \\
		\midrule 
		%Best-BP (10ITER, damping) & 0.22 & 0.09 & 0.20 & 0.18  \\
		BPNN~\cite{kuck2020belief} & 0.21 & 0.25 & 0.20 & 0.33  \\
		FE-NBP (ours) & \textbf{0.11} & 0.09 & 0.03 & \textbf{0.03} \\
		\bottomrule
	\end{tabular}
	}
	\vspace{2mm}
	\caption{Log-scores on UAI-challenge datasets. FE-NBP performs the best on two out of four datasets while achieving competitive performance on the other two datsets. Note that when the output assignment has a ground truth probability of zero, the log-score will be inf.}% (the search algorithms perform poorly on the ObjDetct and DBN datasets because there are many assignments with zero probability in their factor graphs.% with the UAI-challenge datasets}% the intra-category generalizability on the UAI-challenge datasets. 
	%(1) The metric is $\frac{1}{|\mathcal{D}|}\sum_{G\in\mathcal{D}}\left|\frac{\log\text{score}(x^*)-\log\text{score}(\hat x^*)}{\log\text{score}(x^*)}\right|$ where $x^*$ is the true MAP state provided by the dataset and $\hat{x}^*$ is the estimated MAP state. The lower the better.
	%}
	\label{tab:intra-uai}
	\vspace{-5mm}
\end{table}

\section{Related Works}
%in the past, GNNs for graphical model inference are only investigated on factor graphs where factors contain at most two variables.

%\paragraph{Belief Propagation Neural Network}
Belief propagation neural network (BPNN)~\cite{kuck2020belief} is a class of inference models that takes a factor graph as input and estimates factor graph's log partition function. BPNN strictly generalizes (sum-product) BP and still guarantees to give a lower bound to the partition function upon convergence for a class of factor graphs by finding fixed points of BP\footnote{For lack of space, find proofs at \cite{kuck2020belief}.}. % Like BP, BPNN preserves the symmetries inherent to factor graphs (Definition~\ref{def:fg_isomorphism_}).
%BPNN iterative layers are neural operators that operate on beliefs or message in a variety of ways. 
More specifically, BPNN keeps the variable-to-factor messages as it is in Equation~\ref{eq:varToFacMsgs_log} but modifies factor-to-variable messages (Equation~\ref{eq:facToVarMsgs_log}) using the output of a learned operator. %: $m_{a \rightarrow i}^{(k)} = \tilde{m}_{a \rightarrow i}^{(k)} + \text{MLP}\big(m^{(k-1)}_{a \rightarrow i} - \tilde{m}^{(k)}_{a \rightarrow i}\big)$. \tang{I am not sure it is okay to directly state MLP in BPNN. Probably @Jonathan can verify it.} 
Similar to the idea of correcting belief propagation's outputs by a learned neural network module, \cite{satorras2021neural} proposed Neural Enhanced Belief Propagation (NEBP), a hybrid model that runs conjointly a GNN with belief propagation. The GNN receives as input messages from belief propagation at every inference iteration and outputs a calibrated version of them. However, these existing works that attempt to combine belief propagation and neural network fail to take Variable Assignment Symmetry of factor graph isomorphism into consideration. In this paper, we address this issue by proposing FE-NBP.% In other words, these algorithms are not equivariant to the permutation of variable assignments (since MLP is not equivariant to the permutation of its input vector). In this paper, we propose FE-NBP to address this issue. The key idea is that instead of learning a neural network modules that adjusts message vectors, we learn a neural network that adjusts message vectors at every dimension \textit{individually}. Empirically, we show that FE-NBP outperforms BP, NEBP, and BPNN in the Ising model experiment and on the UAI-challenge dataset.%of message vectors. 

%\jonathan{would be good to cite \url{https://openaccess.thecvf.com/content_CVPR_2020/html/Knobelreiter_Belief_Propagation_Reloaded_Learning_BP-Layers_for_Labeling_Problems_CVPR_2020_paper.html}}

Instead of augmenting BP with neural networks, some other works aim to devise end-to-end trainable inference systems. In \cite{Yoon2018InferenceIP}, they apply Gated Graph Neural Network (GGNN)~\cite{li2015gated} to graphical model inference. Factor Graph Neural Network (FGNN)~\cite{zhang2019factor} is another graph neural network model proposed to perform MAP inference on factor graphs. Mimicing the procedures of max-product belief propagation, FGNN consists of Variable-to-Factor modules and a Factor-to-Variable modules. Nevertheless, these works fail to consider Local Variable Symmetry and Variable Assignment Symmetry. In this paper, we propose FE-GNN, which considers one more symmetry than existing GNN-based inference models.%Their graph neural networks is not equivariant or invariant to the permutation of variable orderings within factors and the variable assignments. \tang{Previously have both variable orderings within factors and the variable assignments. However, later, only FE-GNN. Seems like no corresponce?} %On the other hand, FE-GNN is equivariant to variable orderings within factors. The key insight is that by incorporating tensor sum operation and summing tensors at the corresponding dimensions, we can make FE-GNN equivariant to permutation of variable orderings. Additionally, FE-GNN can handle any arbitrary factor graph whereas existing GNN-based work like \cite{Yoon2018InferenceIP} are mostly designed to handle ``symmetric'' binary markov random fields. We demonstrate that FE-GNN is superior to GGNN and FGNN in both the Ising model experiment and the asymmetric binary markov random field experiment.

\section{Conclusion}
In this paper, we identify factor graph isomorphism and introduce two neural network-based inference models that takes advantages of such inductive biases: Factor-Equivariant Neural Belief Propagation (FE-NBP) and Factor-Equivariant Graph Neural Networks (FE-GNN). FE-NBP is an inference model that incorporates a learnable neural network module on top of BP while respecting all symmetries of factor graphs. FE-GNN is an end-to-end trainable GNN model that has great expressivity while respecting one more symmetry than existing GNN-based models. We perform experiments on both marginal inference and MAP inference and show that FE-NBP and FE-GNN achieves state-of-the-art performance on different sample complexity regimes. We further perform experiments to support that the proposed inductive biases are indeed beneficial for neural network-based inference models.%datasets with more 1000 training instances. We further show that FE-NBP can learn on factor graphs with at least one factor potentials being asymmetric tensors whereas other existing GNN-based inference models cannot.%  with n both real-world and synthetic dataset and show that FE-NBP achieves state-of-the-art performance on small sample-sized datasets while FE-GNN achieves state-of-the-art performance on large sample-sized datasets.% excels in smaller sample-sized dataset and FE-GNN achieves state-of-the-art performance in large sample-sized datset.

%\bibliographystyle{plain}
%\bibliography{neurips_2021}

\begin{thebibliography}{10}

\bibitem{baxter2016exactly}
Rodney~J Baxter.
\newblock {\em Exactly solved models in statistical mechanics}.
\newblock Elsevier, 2016.

\bibitem{braunstein2004survey}
Alfredo Braunstein and Riccardo Zecchina.
\newblock Survey propagation as local equilibrium equations.
\newblock {\em Journal of Statistical Mechanics: Theory and Experiment},
  2004(06):P06007, 2004.

\bibitem{chandler1987introduction}
David Chandler.
\newblock Introduction to modern statistical mechanics.
\newblock {\em Oxford University Press, Oxford, UK}, 1987.

\bibitem{chen2020learning}
Yihao Chen, Xin Tang, Xianbiao Qi, Chun-Guang Li, and Rong Xiao.
\newblock Learning graph normalization for graph neural networks.
\newblock {\em arXiv preprint arXiv:2009.11746}, 2020.

\bibitem{chung2014empirical}
Junyoung Chung, Caglar Gulcehre, KyungHyun Cho, and Yoshua Bengio.
\newblock Empirical evaluation of gated recurrent neural networks on sequence
  modeling.
\newblock {\em arXiv preprint arXiv:1412.3555}, 2014.

\bibitem{Dept.2015}
Carnegie-Mellon University.Computer~Science Dept.
\newblock {Speech understanding systems: summary of results of the five-year
  research effort at Carnegie-Mellon University.}
\newblock 4 2015.

\bibitem{gal2006residual}
Gal Elidan, Ian McGraw, and Daphne Koller.
\newblock Residual belief propagation: Informed scheduling for asynchronous
  message passing.
\newblock In {\em Proceedings of the Twenty-Second Conference on Uncertainty in
  Artificial Intelligence}, UAI'06, page 165–173, Arlington, Virginia, USA,
  2006. AUAI Press.

\bibitem{gilmer2017neural}
Justin Gilmer, Samuel~S Schoenholz, Patrick~F Riley, Oriol Vinyals, and
  George~E Dahl.
\newblock Neural message passing for quantum chemistry.
\newblock In {\em International Conference on Machine Learning}, pages
  1263--1272. PMLR, 2017.

\bibitem{globerson2007fixing}
Amir Globerson and Tommi Jaakkola.
\newblock Fixing max-product: Convergent message passing algorithms for map
  lp-relaxations.
\newblock {\em Advances in neural information processing systems}, 20:553--560,
  2007.

\bibitem{gori2005new}
Marco Gori, Gabriele Monfardini, and Franco Scarselli.
\newblock A new model for learning in graph domains.
\newblock In {\em Proceedings. 2005 IEEE International Joint Conference on
  Neural Networks, 2005.}, volume~2, pages 729--734. IEEE, 2005.

\bibitem{kingma2014adam}
Diederik~P Kingma and Jimmy Ba.
\newblock Adam: A method for stochastic optimization.
\newblock {\em arXiv preprint arXiv:1412.6980}, 2014.

\bibitem{kipf2016semi}
Thomas~N Kipf and Max Welling.
\newblock Semi-supervised classification with graph convolutional networks.
\newblock {\em arXiv preprint arXiv:1609.02907}, 2016.

\bibitem{koller2009probabilistic}
Daphne Koller and Nir Friedman.
\newblock {\em Probabilistic graphical models: principles and techniques}.
\newblock MIT press, 2009.

\bibitem{kschischang2001factor}
Frank~R Kschischang, Brendan~J Frey, and H-A Loeliger.
\newblock Factor graphs and the sum-product algorithm.
\newblock {\em IEEE Trans. on information theory}, 47(2):498--519, 2001.

\bibitem{kuck2020belief}
Jonathan Kuck, Shuvam Chakraborty, Hao Tang, Rachel Luo, Jiaming Song, Ashish
  Sabharwal, and Stefano Ermon.
\newblock Belief propagation neural networks.
\newblock {\em arXiv preprint arXiv:2007.00295}, 2020.

\bibitem{lauritzen1988local}
Steffen~L Lauritzen and David~J Spiegelhalter.
\newblock Local computations with probabilities on graphical structures and
  their application to expert systems.
\newblock {\em Journal of the Royal Statistical Society: Series B
  (Methodological)}, 50(2):157--194, 1988.

\bibitem{li2015gated}
Yujia Li, Daniel Tarlow, Marc Brockschmidt, and Richard Zemel.
\newblock Gated graph sequence neural networks.
\newblock {\em arXiv preprint arXiv:1511.05493}, 2015.

\bibitem{mackay1999good}
David~JC MacKay.
\newblock Good error-correcting codes based on very sparse matrices.
\newblock {\em IEEE transactions on Information Theory}, 45(2):399--431, 1999.

\bibitem{merle2019turing}
M{\'e}lody Merle, Laura Messio, and Julien Mozziconacci.
\newblock Turing-like patterns in an asymmetric dynamic ising model.
\newblock {\em Physical Review E}, 100(4):042111, 2019.

\bibitem{mezard2002analytic}
Marc M{\'e}zard, Giorgio Parisi, and Riccardo Zecchina.
\newblock Analytic and algorithmic solution of random satisfiability problems.
\newblock {\em Science}, 297(5582):812--815, 2002.

\bibitem{Mooij_libDAI_10}
Joris~M. Mooij.
\newblock lib{DAI}: A free and open source {C++} library for discrete
  approximate inference in graphical models.
\newblock {\em JMLR}, 11:2169--2173, August 2010.

\bibitem{Murphy1999LoopyBP}
Kevin~P. Murphy, Yair Weiss, and Michael~I. Jordan.
\newblock Loopy belief propagation for approximate inference: An empirical
  study.
\newblock In {\em UAI}, 1999.

\bibitem{neal1993probabilistic}
Radford~M Neal.
\newblock {\em Probabilistic inference using Markov chain Monte Carlo methods}.
\newblock Department of Computer Science, University of Toronto Toronto,
  Ontario, Canada, 1993.

\bibitem{mcbook}
Art~B. Owen.
\newblock Monte carlo theory, methods and examples, 2013.

\bibitem{pearl1984heuristics}
Judea Pearl.
\newblock Heuristics: intelligent search strategies for computer problem
  solving.
\newblock 1984.

\bibitem{satorras2020neural}
Victor~Garcia Satorras and Max Welling.
\newblock Neural enhanced belief propagation on factor graphs.
\newblock {\em arXiv preprint arXiv:2003.01998}, 2020.

\bibitem{satorras2021neural}
Victor~Garcia Satorras and Max Welling.
\newblock Neural enhanced belief propagation on factor graphs.
\newblock In {\em International Conference on Artificial Intelligence and
  Statistics}, pages 685--693. PMLR, 2021.

\bibitem{scarselli2008graph}
Franco Scarselli, Marco Gori, Ah~Chung Tsoi, Markus Hagenbuchner, and Gabriele
  Monfardini.
\newblock The graph neural network model.
\newblock {\em IEEE transactions on neural networks}, 20(1):61--80, 2008.

\bibitem{wainwright2008graphical}
Martin~J Wainwright, Michael~I Jordan, et~al.
\newblock Graphical models, exponential families, and variational inference.
\newblock {\em Foundations and Trends{\textregistered} in Machine Learning},
  1(1--2):1--305, 2008.

\bibitem{xu2018powerful}
Keyulu Xu, Weihua Hu, Jure Leskovec, and Stefanie Jegelka.
\newblock How powerful are graph neural networks?
\newblock In {\em ICLR}, 2018.

\bibitem{xu2018representation}
Keyulu Xu, Chengtao Li, Yonglong Tian, Tomohiro Sonobe, Ken-ichi Kawarabayashi,
  and Stefanie Jegelka.
\newblock Representation learning on graphs with jumping knowledge networks.
\newblock In {\em International Conference on Machine Learning}, pages
  5453--5462. PMLR, 2018.

\bibitem{yedidia2005constructing}
Jonathan~S Yedidia, William~T Freeman, and Yair Weiss.
\newblock Constructing free-energy approximations and generalized belief
  propagation algorithms.
\newblock {\em IEEE Trans. on information theory}, 51(7):2282--2312, 2005.

\bibitem{Yoon2018InferenceIP}
KiJung Yoon, Renjie Liao, Yuwen Xiong, Lisa Zhang, Ethan Fetaya, Raquel
  Urtasun, Richard~S. Zemel, and Xaq Pitkow.
\newblock Inference in probabilistic graphical models by graph neural networks.
\newblock {\em ArXiv}, abs/1803.07710, 2018.

\bibitem{zhang2019factor}
Zhen Zhang, Fan Wu, and Wee~Sun Lee.
\newblock Factor graph neural network.
\newblock {\em arXiv preprint arXiv:1906.00554}, 2019.

\end{thebibliography}
%\newpage

%\input{checklist}
\newpage
\appendix 
\section{Factor Graph Isomorphism}
%\vspace{-2mm}
A factor graph is represented as\footnote{Note that a factor graph can be viewed as a weighted hypergraph where factors define hyperedges and factor potentials define hyperedge weights for every variable assignment within the factor.} $G = (A, P, I, X)$.  %
%$G = (F, V, A, P, I)$.  $F$ represents an ordered list of $M$ factor nodes.  $V$ represents an ordered list of $N$ variable nodes.  
$A \in \{0,1\}^{M \times N}$ is an adjacency matrix over $M$ factor nodes and $N$ variable nodes, where $A_{ai}=1$ if the i-th variable is in the scope of the a-th factor and $A_{ai}=0$ otherwise. $P$ is an ordered list of $M$ factor potentials, where the a-th factor potential, $P_a$, corresponds to the a-th factor (row) in $A$ and is represented as a tensor with one dimension for every variable in the scope of $P_a$. $I$ is an ordered list of ordered lists that locally indexes variables within each factor. $I_a$ is an ordered list specifying the local indexing of variables within the a-th factor. $I_{ak} = i$ specifies that the k-th dimension of the tensor $P_a$ corresponds to the i-th variable (column) in $A$. $X$ is an ordered list of $N$ variables that specifies possible states of all variables. $X_i$ is a list specifying the local indexing of states within the i-th variable. $X_{ij} = x$ specifies that the j-th state of variable $i$ is $x$. According to our definition, $P_a$ will have the shape of $\bigg(\left|X_{I_{a1}}\right|, \left|X_{I_{a2}}\right|, \cdots, \left|X_{I_{aK}}\right|\bigg)$ where $K$ is the number of variables associated with factor $a$ and $\left|X_i\right|$ is the number of states of variable $i$. We will use $P_a(X_{i}=x)$ to denote the resulting factor potential tensor of $P_a$ when $X_i = x$. % In this paper, we will use $X_i(j)$ to denote the j-th position state of variable $i$.

Two factor graphs is isomorphic when they meet the following conditions:%In this paper, we will define two types of factor graphs isomorphism:% when they meet the conditions of the following definition: %~\ref{def:fg_isomorphism_}.
%We define two factor graphs to be isomorphic when they meet the conditions of Definition~\ref{def:fg_isomorphism_}.

\begin{definition}[Factor Graph Isomorphism] \label{def:fg_isomorphism_}
Factor graphs $G = (A, P, I, X)$ and $G' = (A', P', I', X')$ with $A \in \{0,1\}^{M \times N}$ and $A' \in \{0,1\}^{M' \times N'}$ are isomorphic if and only if %$|G(F)| = |G'(F)| = M$, $|G(V)| = |G'(V)| = N$, and
$M=M'$, $N=N'$, and
\begin{enumerate}
    \item  There exist bijections\footnote{For $K \in \mathbb{N}$, we use $[K]$ to denote $\{1, 2, \ldots, K\}$.}
    $f_F: [M] \to [M]$ and $f_V: [N] \to [N]$
    such that $A_{ai} = A'_{bj}$ for all $a \in [M]$ and $i \in [N]$, where $b=f_F(a)$ and $j=f_V(i)$.
   
    \item There exists a bijection $F_a : [|I_a|]\rightarrow [|I'_b|]$ for every factor $a \in [M]$ in $G$ and factor $b= f_F(a)$ in $G'$ such that $f_V\big(I_{ak}\big) = I'_{bl}$ \, and $P_a = \sigma_a\big(P'_b\big)$ for all $k \in [|I_i|]$, where $l=F_a(k)$, and $\sigma_a\big(P'_b\big)$ denotes permuting the dimensions of the tensor $P'_b$ to the order of $\big((F_a(I_{a1}), F_a(I_{a2}), \dots, F_a(I_a(|I_a|))\big)$.
    %\tang{Comment: I have some rough ideas of what you are trying to say. But, unfortunately, I still cannot understand $\delta_a$ after the second glance, probably because the denotations of $f_a^{idx}$ and $G(*)$ are not so straightforward?  }
    \item There exists a bijection $F_i: [|X_i|] \rightarrow [|X'_j|]$ for every variable $i$ in $G$ and variable $j = f_V(i)$ in $G'$, such that for all factors $a$ in $G$ that associates with variable $i$ and factor $b = f_F(a)$ in $G'$, $P_a(X_{i}=k) = \sigma(P'_{b})(X_{j}=l)$ for all $k \in [|X_i|]$, where $l=F_i(k)$ and $\sigma(P'_{b})$ denotes permuting the order of variables states of all variables $v$ in $\sigma_a(P'_b)$ according to $\big(F_v(X_{v0}), \dots, F_v(X_{v|X_i|})\big)$.% states of $X_j$ according to $X_i$. %$P'_{f_F(a)}(X_{f_V(i)}=l)$ according to $\sigma_a$.
\end{enumerate}
\end{definition}

\section{Factor graph isomorphism for FE-NBP}
%In detail, as stated previously, a factor graph is denoted as $G=(A, F^p, F^{idx})$ where the factor potentials of the factor $a$ is stored in a tensor $F^p_a$ of shape $\left[\left|X_{F^{idx}_a(1)}\right|, \left|X_{F^{idx}_a(2)}\right|, \cdots, \left|X_{F^{idx}_a(K)}\right|\right]$, $K$ is the number of variables in the factor $a$, $F^{idx}_a(k)$ is the index of variable that corresponds to the $k$th dimension of the factor potential tensor $F^p_a$, and $|X_{i}|$ denotes the cardinality, i.e., the number of states, of the variable indexed by $i$. We also denote the permutation group of the states of the variable $i$ as $PG_i=S_{|X_i|}$ where $S_{|X_i|}$ includes all the possible permutations of $|X_i|$ elements. The permutation group of the states of all variables is then $PG=PG_1\times PG_2\times \cdots PG_N=S_{|X_1|}\times S_{|X_2|}\times \cdots S_{|X_N|}$.

We denote the permutation group of of variable orderings within the $i$-th factor as $PG_i$, the permutation group of variable orderings of all factors $PG=PG_1\times PG_2\times \cdots PG_M=S_{|I_1|}\times S_{|I_2|}\times \cdots S_{|I_M|}$ %\tang{May need to first define PG?, also $\delta(G)$ and $\delta(\mathcal{F}(G))$? (I remembered that I have written one definition before.} 
where $S_{|I_i|}$ includes all the possible permutations of $|I_i|$ elements.
To satisfy Local Variable Symmetry of factor graph isomorphism in Definition~\ref{def:factor_graph_isomorphism}, the output of an inference algorithm $\mathcal{F}$ that operates on a factor graph $G$ should be equivariant (or invariant) to the permutation group $PG$ of variable orderings within factors. That means for any permutation $\sigma\in PG$, $\mathcal{F}(\sigma(G))=\sigma(\mathcal{F}(G))$. In more details, the permutation $\sigma$ is composed of the permutations of all variable orderings within factors. In more details, a permutation $\sigma$ is composed of permutations of variables within all factors, i.e., $\sigma=\sigma_1\times \sigma_2\times\cdots\sigma_M$ and $\sigma_a\in S_{|I_a|}$ ($|I_a$ is the number of variables associated with factor $a$).
%, i.e., $A'=A$, $F^{idx}'=F^{idx}$, and $F^p_a'[\sigma_{F^{idx}_a(1)}(i_1), \sigma_{F^{idx}_a(2)}(i_2), \cdots, \sigma_{F^{idx}_a(K)}(i_K)]=F^p_a[i_1,i_2,\cdots,i_K]$. $\sigma(\mathcal{F_A}(G))$ denotes the inference outputs after permutation, i.e., for the marginal inference task, the permutated marginals $m'_i[\sigma_i(j)]=m_i[j]$ for $i\in[N]$ and $j\in |X_i|$, and for the MAP inference task, the permutated decoded assignment $\hat{x}_i^*'=\sigma_i^{-1}(\hat{x}_i^*)$.
%It is easy to verify that FE-NBP is equivariant to the permutation group $PG$ since each component of Factor BPNN is equivariant to the permutation group $PG$. It is also clear that BPNN and, later on, Factor Graph Neural Networks do not respect this kind of factor-graph isomorphism.

Writing down all steps in FE-NBP, we have:
\begin{equation}
\begin{gathered}
m_{i \rightarrow a}^{(k)}(x_l) = -z_{i \rightarrow a} + \sum_{c \in \mathcal{N}(i) \setminus a} m_{c \rightarrow i}^{(k-1)}(x_l), \\
m_{a \rightarrow i}^{(k)}(x_l) = \tilde{m}_{a \rightarrow i}^{(k)}(x_l) +  \alpha^{(k)}_{a \rightarrow i}(x_l) \big(m^{(k-1)}_{a \rightarrow i}(x_l) - \tilde{m}^{(k)}_{a \rightarrow i}(x_l)\big), \\
 \tilde{m}_{a \rightarrow i}^{(k)}(x_l) = -z_{a \rightarrow i} +  \LSE{\mathbf{x}_a \setminus X_i=x_l} \bigg(\Psi_a(\mathbf{x}_a) + \sum_{j \in \mathcal{N}(a) \setminus i} m_{j \rightarrow a}^{(k)}(x_j) \bigg), \\
\alpha^{(k)}_{a \rightarrow i}(x_l) = \phi_{\text{NN}}\left({m}^{(k-1)}_{a \rightarrow i}(x_l), \tilde{m}^{(k)}_{a \rightarrow i}(x_l), b_i^{(k)}(x_l), \sum_{\mathbf{x}_a\backslash X_i=x_l}b_a^{(k)}(\mathbf{x}_a), \max_{\mathbf{x}_a\backslash X_i=x_l}b_a^{(k)}(\mathbf{x}_a)\right).\\
\end{gathered}
\end{equation}
Observe that no terms in the FE-NBP's formulations is affected by the permutation of variables within factors. Every term in the above formulations is a scalar. $\phi_{\text{NN}}$ only takes features that are relevant to variable $i$ and its variable assignment $x_l$, thus it is invariant to variable orderings. % Note that all other variables except $i$ are  ``variable belief'' feature and 
That is, $\mathcal{F}(\sigma(G))=\sigma(\mathcal{F}(G))$ when $\mathcal{F}$ corresponds to FE-NBP since $\sigma$ does not affect any step in the algorithm. FE-NBP respects Local Variable Symmetry of factor graph isomorphism in the same way that BP respects Local Variable Symmetry. For proof of BP respecting Local Variable Symmetry of factor graph isomorphism, refer to the Appendix section in \cite{kuck2020belief}.% any permutation $\sigma$. of variables within factors is $\Psi$, the factor potentials.

To satisfy Variable Assignment Symmetry of factor graph isomorphism in Definition~\ref{def:factor_graph_isomorphism}, the output of an inference algorithm $\mathcal{F}$ that operates on a factor graph $G$ should be equivariant or invariant to the permutation group of all variable assignment orderings. Let $PG$ be the permutation group of all variable assignment orderings. A permutation $\sigma \in PG$ is composed of the permutations of all variables' assignments, i.e., $\sigma=\sigma_1\times \sigma_2\times\cdots\sigma_N$ and $\sigma_i\in S_{|X_i|}$ ($|X_i$ is the number of possible assignments of variable $i$). Let $G'=\sigma(G)$ denote the factor graph after the permutation, we proceed to show that for any permutation $\sigma\in PG$, $\mathcal{F}(\sigma(G))==\sigma(\mathcal{F}(G))$ when $\mathcal{F}$ corresponds to FE-NBP. $\mathcal{F}(G)$ is the output of $F$ when given an input factor graph $G$.% For marginal inference tasks, the output of $\mathcal{F}$ a list of vectors, one for each variable.
\begin{proof}
%\begin{lemma}\label{lemma:msg_equiv}
%\textbf{Message equivariance:} 

Let $m^{(k)}_{i \rightarrow a}$ and $m'^{(k)}_{i \rightarrow a}$ denote variable to factor messages and $m^{(k)}_{a \rightarrow i}$ and $m'^{(k)}_{a \rightarrow i}$ factor to variable messages obtained by applying k iterations of FE-NBP to factor graphs $G$ and $G'$. Our ultimate goal is to show $\mathcal{F}(\sigma(G))==\sigma(\mathcal{F}(G))$. Alternatively, we can try to prove that at any iteration k, $m^{(k)}_{i \rightarrow a}(\sigma(x_l)) = m'^{(k)}_{i \rightarrow a}(x_l)$ and $m^{(k)}_{a \rightarrow i}(\sigma(x_l)) = m'^{(k)}_{a \rightarrow i}$ since message vectors are the output of FE-NBP.%  If $G$ and $G'$ are isomorphic as factor graphs and messages are initialized to a constant\footnote{Any message initialization strategy can be used, as long as initial messages are equivariant; e.g. they satisfy the bijective mapping  $g^{(0)}_{i \rightarrow a} = h^{(0)}_{j \rightarrow b}$ and $g^{(k)}_{a \rightarrow i} = h^{(k)}_{b \rightarrow j}$ where $j=f_V(i)$ and $b=f_F(a)$.}, then there is a bijective mapping between messages: $g^{(k)}_{i \rightarrow a} = h^{(k)}_{j \rightarrow b}$ and $g^{(k)}_{a \rightarrow i} = h^{(k)}_{b \rightarrow j}$ where $j=f_V(i)$ and $b=f_F(a)$.% This property holds for BPNN-D iterative layers if $H(\cdot)$ is equivariant to global node indexing. 
~Note that we slightly abuse the notation by using $\sigma$ to refer to a permutation and the bijective mapping determined by the permutation.

Base case: the initial messages are all equal when constant initialization is used and therefore satisfy any bijective mapping.

Assume $m^{(k-1)}_{i \rightarrow a}(\sigma(x_l)) = m'^{(k-1)}_{i \rightarrow a}(x_l)$ and $m^{(k-1)}_{a \rightarrow i}(\sigma(x_l)) = m'^{(k-1)}_{a \rightarrow i}$ hold for all variable assignments $x_l$ at iteration $k-1$.

Inductive step: by our assumption and Equation~\ref{eq:varToFacMsgs_log} (the definition of variable to factor messages), we have %\tang{Not so important. Just a reminder that if you have a more formal definition of $\delta$ on the factor graph, beliefs, and messages, the following proof would be much easier and straightforward.}
\begin{align} \label{eq:ind-1}
\begin{split}
     m_{i \rightarrow a}^{(k)}(\sigma(x_l)) &=  -z_{i \rightarrow a} + \sum_{c \in \mathcal{N}(i) \setminus a} m_{c \rightarrow i}^{(k-1)}(\sigma(x_l)) \\
     &= -z_{i \rightarrow a} + \sum_{c \in \mathcal{N}(i) \setminus a} m'^{(k-1)}_{c \rightarrow i}(x_l) = m'^{(k)}_{i \rightarrow a}(x_l). \\
\end{split}
\end{align}
%Note that none of the terms in Equation~\ref{eq:ind-1} are affected by the permutation $\sigma$.% does not have any effect on variable-to-factor nodes since none of the terms are 
%\begin{equation}
%    g_{i \rightarrow a}^{(k)}(x_i) = \prod_{c \in \mathcal{N}(i) \setminus a} g_{c \rightarrow i}^{(k-1)}(x_i) = \prod_{c \in \mathcal{N}(j) \setminus b} h_{c \rightarrow j}^{(k-1)}(x_j) = h^{(k)}_{j \rightarrow b}(x_j),
%\end{equation}
%where $j=f_V(i)$ and $b=f_F(a)$, 
%since the bijective mapping holds for factor to variable messages at iteration $k-1$ by the inductive hypothesis.  
By our assumption and Equation~\ref{eq:facToVarMsgs_log} (the definition of factor to variable messages), we have
\begin{align} \label{eq:ind-2}
\begin{split}
%m_{a \rightarrow i}^{(k)}(\sigma(x_l)) &= \tilde{m}_{a \rightarrow i}^{(k)}(\sigma(x_l)) +  \alpha^{(k)}_{a \rightarrow i}(\sigma(x_l)) \big(m^{(k-1)}_{a \rightarrow i}(\sigma(x_l)) - \tilde{m}^{(k)}_{a \rightarrow i}(\sigma(x_l))\big) \\
 \tilde{m}_{a \rightarrow i}^{(k)}(\sigma(x_l)) &= -z_{a \rightarrow i} +  \LSE{\mathbf{x}_a \setminus X_i=\sigma(x_l)} \bigg(\Psi_a(\mathbf{x}_a) + \sum_{j \in \mathcal{N}(a) \setminus i} m_{j \rightarrow a}^{(k)}(\sigma(x_j)) \bigg) \\
 &= -z_{a \rightarrow i} +  \LSE{\mathbf{x}_a \setminus X_i=x_l} \bigg(\Psi'_a(\mathbf{x}_a) + \sum_{j \in \mathcal{N}(a) \setminus i} m'^{(k)}_{j \rightarrow a}(x_j) \bigg) =  \tilde{m}'^{(k)}_{a \rightarrow i}(x_l)\\
%\alpha^{(k)}_{a \rightarrow i}(\sigma(x_l)) = \phi_{\text{NN}}\left({m}^{(k-1)}_{a \rightarrow i}(\sigma(x_l)), \tilde{m}^{(k)}_{a \rightarrow i}(\sigma(x_l)), b_i^{(k)}(\sigma(x_l)), \sum_{\mathbf{x}_a\backslash X_i=\sigma(x_l)}b_a^{(k)}(\mathbf{x}_a), \max_{\mathbf{x}_a\backslash X_i=\sigma(x_l)}b_a^{(k)}(\mathbf{x}_a)\right).\\
%    g_{a \rightarrow i}^{(k)}(x_i) =  \sum_{\mathbf{x}_a \setminus x_i} G(F^p_a)(\mathbf{x}_a) \prod_{l \in \mathcal{N}(a) \setminus i} g_{l \rightarrow a}^{(k)}(\sigma(x_l)) \\
%    =  \sum_{\mathbf{x}_b \setminus x_j} \sigma_a\big(G'(F^p_b)\big)(\mathbf{x}_b) \prod_{l \in \mathcal{N}(b) \setminus j} g_{l \rightarrow b}^{(k)}(\sigma(x_l)) = g_{b \rightarrow j}^{(k)}(x_j). \\
\end{split} 
\end{align}
Let $x_m$ be the shorthand of $\sigma(x_l)$. By Equation \ref{eq:calculate_beliefs}, \ref{eq:fenbp-2}, \ref{eq:ind-1}, and \ref{eq:ind-2}, we have
\begin{align} \label{eq:ind-3}
\begin{split}
\alpha^{(k)}_{a \rightarrow i}(x_m) &= \phi_{\text{NN}}\left({m}^{(k-1)}_{a \rightarrow i}(x_m), \tilde{m}^{(k)}_{a \rightarrow i}(x_m), b_i^{(k)}(x_m), \sum_{\mathbf{x}_a\backslash X_i=x_m}b_a^{(k)}(\mathbf{x}_a), \max_{\mathbf{x}_a\backslash X_i=x_m}b_a^{(k)}(\mathbf{x}_a)\right) \\
&= \phi_{\text{NN}}\left({m}'^{(k-1)}_{a \rightarrow i}(x_l), \tilde{m}'^{(k)}_{a \rightarrow i}(x_l), b'^{(k)}_i(x_l), \sum_{\mathbf{x}_a\backslash X_i=x_l}b'^{(k)}_a(\mathbf{x}_a), \max_{\mathbf{x}_a\backslash X_i=x_l}b'^{(k)}_a(\mathbf{x}_a)\right) \\
&= \alpha'^{(k)}_{a \rightarrow i}(x_l)
\\
\end{split}
\end{align}
By Equation~\ref{eq:fenbp-2}, \ref{eq:ind-1}, \ref{eq:ind-2}, and \ref{eq:ind-3}, we have
\begin{align}
\begin{split}
m_{a \rightarrow i}^{(k)}(\sigma(x_l)) &= \tilde{m}_{a \rightarrow i}^{(k)}(\sigma(x_l)) +  \alpha^{(k)}_{a \rightarrow i}(\sigma(x_l)) \big(m^{(k-1)}_{a \rightarrow i}(\sigma(x_l)) - \tilde{m}^{(k)}_{a \rightarrow i}(\sigma(x_l))\big) \\
&= \tilde{m}'^{(k)}_{a \rightarrow i}(x_l) +  \alpha'^{(k)}_{a \rightarrow i}(x_l) \big(m'^{(k-1)}_{a \rightarrow i}(x_l) - \tilde{m}'^{(k)}_{a \rightarrow i}(x_l)\big) = m'^{(k)}_{a \rightarrow i}(x_l).
%\alpha^{(k)}_{a \rightarrow i}(\sigma(x_l)) = \phi_{\text{NN}}\left({m}^{(k-1)}_{a \rightarrow i}(\sigma(x_l)), \tilde{m}^{(k)}_{a \rightarrow i}(\sigma(x_l)), b_i^{(k)}(\sigma(x_l)), \sum_{\mathbf{x}_a\backslash X_i=\sigma(x_l)}b_a^{(k)}(\mathbf{x}_a), \max_{\mathbf{x}_a\backslash X_i=\sigma(x_l)}b_a^{(k)}(\mathbf{x}_a)\right).\\
\end{split}
\end{align}
showing that the bijective mapping continues to hold at iteration $k$.
Therefore, we prove that for any $k\ge 1$, the outputs of FE-NBP is equivalent to the permutation of the orderings of variable assignments... i.e., $\mathcal{F}(\delta(G))\equiv \delta(\mathcal{F}(G))$...

%Proof extension to BPNN-D: the logic of the proof is unchanged when BP is performed in log-space with damping.  The only difference between BPNN-D and standard BP is the replacement of the term $\alpha \big(\overline{m}_{a \rightarrow i}^{(k-1)} - \tilde{m}_{a \rightarrow i}^{(k)}\big)$ in the computation of factor to variable messages with $\Delta^{(k)}_{a \rightarrow i}$, where $\Delta^{(k)} = H\big(\overline{n}^{(k-1)} - \tilde{n}^{(k)}\big)$.  If $H(\cdot)$ is equivariant to global node indexing (the bijective mapping $\Delta^{(k)}_{a \rightarrow i}(G) = \Delta^{(k)}_{b \rightarrow j}(G')$ holds, where $\Delta^{(k)}_{a \rightarrow i}(G)$ denotes applying the operator $H(\cdot)$ to the k-th iteration's message differences when the input factor graph is $G$ and taking the output correpsonding to message $a \rightarrow i$), then equality is maintained in Equation~\ref{eq:FtoV_msg_equiv} and the bijective mapping between messages holds.
\end{proof}

\section{Upper bound and lower bound of FE-NBP on MAP inference}

\begin{prop}
FE-NBP can provide an upper bound and a lower bound for the probability / log-score of the MAP assignment.
\end{prop} 

\begin{proof}
Upper bound:% (Note that estimating the partition function $Z$ is NP-hard. To be more efficient, people usually use the bounds of the log-scores as defined in Section 4.3.):
\begin{eqnarray*}
p(x^*) & = & \exp\left(-\log Z+\sum_{a=1}^M \log f_a(\mathbf{x}^*_a)\right) 
\text{~~~~~~~\small(by definition)}\\
& = & \exp\left(-\log Z+\sum_{a=1}^M \Psi_a(\mathbf{x}^*_a)\right) 
\text{~~~~~~~\small(change the notation)}\\
& = & \exp\left(-\log Z+\sum_{a=1}^M \left(\Psi_a(\mathbf{x}^*_a)-\sum_{j\in\mathcal{N}(a)}m_{a\rightarrow j}^{(k)}(x^*_j)\right)+\sum_{i=1}^N\sum_{b\in\mathcal{N}(i)}m_{b\rightarrow j}^{(k)}(x_i^*)\right)\\
&&\text{~~~~~~~~~~~~~~~~~~~~~~~~~~~~~~~~~~~~~~~~~~~~~~~~~~~~~~~~~~~~~~~~~~~~~~~~~~~~~~~~~~~~~~~~~~~~~~~~~~~~~~~~~~~~~~~~~~~~~~~~\small(reformulation)}\\
& \le & \exp\left(-\log Z+\sum_{a=1}^M \max_{\mathbf{x}_a}\left(\Psi_a(\mathbf{x}_a)-\sum_{j\in\mathcal{N}(a)}m_{a\rightarrow j}^{(k)}(x_j)\right)+\sum_{i=1}^N\max_{x_i}\sum_{b\in\mathcal{N}(i)}m_{b\rightarrow j}^{(k)}(x_i)\right) 
\\
&&\text{~~~~~~~~~~~~~~~~~~~~~~~~~~~~~~~~~~~~~~~~~~~~~~~~~~~~~~~~~~~~~~~~~~~~~~~~~~~~~~~~~~~~~~~~~~~~~~~~~~~~~~~~~~~~~~~~~~~~~~~~\small(by argmax)}\\
\end{eqnarray*}

lower bound:
\begin{eqnarray*}
p(x^*) \ge p(\hat x^*)
\text{~~~~~~~\small(by definition)}
\end{eqnarray*}
where $\hat x^*=\arg\max b_i(x_i)$.
\end{proof}
It is easy to find that the previous proof can be applied to any message-based inference algorithms. Traditionally, people aimed at making the bounds tighter or making the convergence faster.% One benefit of this proposition is that FE-NBP can be applied as an inner procedure of other exact-inference algorithms such as branch-and-bound.

\section{Details of the experiment on Ising models}
In this experiment, $\phi_{\text{NN}_1}$ and $\phi_{\text{NN}_2}$ are parameterized by a GRUs with a hidden dimension of 5. All MLPs in Equation~\ref{eq:fe-gnn1} and \ref{eq:fe-gnn2} have two hidden layers with 64 units each, and use ReLU nonlinearities. In BP of BP with damping, message propagates for at most 200 steps. In all neural network-based inference models, messages propagate for $T=10$ time steps. All inference procedures with a neural network are optimized on binary cross-entropy loss of estimated marginals and ground truth, trained with ADAM~\cite{kingma2014adam} with a learning rate of 0.001. We use early stopping with a window size of 5. Results have been averaged over two runs.
%In our experiments, we only use the ``sum factor belief'' feature $\left(\sum_{\mathbf{x}_a\backslash X_i=x_l}b_a^{(k)}(\mathbf{x}_a)\right)$ when FE-NBP is used for marginal inference and ``max factor belief'' feature $\left(\max_{\mathbf{x}_a\backslash X_i=x_l}b_a^{(k)}(\mathbf{x}_a)\right)$ when FE-NBP is used to perform MAP inference. 

\section{Details of the experiment on UAI-challenge datasets}

\paragraph{Model Details} In our experiments, $\phi_{\text{NN}}$ is parameterized by a three-layer MLP with graph-wise normalization~\cite{chen2020learning} before each activation function, the leaky ReLU. We train FE-NBP with the Adam optimizer, learning rate as 0.0001, and the number of hidden neurons as 64 for 1000 epochs. We tune two hyper-parameters of FE-NBP, which are the utilization of the graph-wise normalization layers and the initialized damping ratios, %(totally 22 choices) 
and report the best performance. We train BPNN with more computation powers and tune its hyper-parameters including the learning rate, the number of neurons, and the initialized damping ratio.

We train our models to minimize the expectation of the UAI loss as defined in the main body. In detail, we first calculate the variable beliefs using the messages updated by our models, as follows:
\begin{eqnarray*}
b^{(K)}_i(x_i) = \frac{1}{Z_i^{(K)}}\exp\left(\sum_{a\in\mathcal{N}(i)}m_{a\rightarrow i}^{(K)}(x_i)\right),
\end{eqnarray*}
where $Z_i^{(K)}$ is the normalization term defined as 
\begin{eqnarray*}
Z^{(K)}_i = \sum_{x_i}\exp\left(\sum_{a\in\mathcal{N}(i)}m_{a\rightarrow i}^{(K)}(x_i)\right).
\end{eqnarray*}
Assuming that we estimate the MAP assignment by sampling from the categorical distribution determined by the variable beliefs, i.e., $p(\hat x^*_i=l)=b_i^{(K)}(l)$, the expectation of the log-score of the estimated MAP assignment is then 
\begin{eqnarray*}
E_{\hat x_i^*\sim b_i^{(K)}}\left[\log\text{score}(\hat x^*)\right] & = & E_{\hat x_i^*\sim b_i^{(K)}}\left[\sum_{a=1}^M\log f_a(\hat{\mathbf{x}}^*_a)\right] \\
& = & \sum_{a=1}^M\sum_{\mathbf{x}_a}\prod_{j\in\mathcal{N}(a)}b_j^{(K)}(x_j)\log f_a({\mathbf{x}}_a)
\end{eqnarray*}
and the training loss is then 
\begin{eqnarray*} 
L & = & \frac{1}{|\mathcal{D}|}\sum_{G\in\mathcal{D}}\left|\frac{\log\text{score}(x^*)-E_{\hat x_i^*\sim b_i^{(K)}}\left[\log\text{score}(\hat x^*)\right]}{\log\text{score}(x^*)}\right|\\
& = & \frac{1}{|\mathcal{D}|}\sum_{G\in\mathcal{D}}\left|\frac{\log\text{score}(x^*)-\sum_{a=1}^M\sum_{\mathbf{x}_a}\prod_{j\in\mathcal{N}(a)}b_j^{(K)}(x_j)\log f_a({\mathbf{x}}_a)}{\log\text{score}(x^*)}\right|.
\end{eqnarray*}

\paragraph{Graph Normalization}
%(Comment: I remember I stated the details of graph-wise normalization previously in the method section. Is there anything wrong with these descriptions?)
The graph-wise normalization operates the same as other normalization layers except for the group of hidden features used for calculating the mean and variance: 
\begin{eqnarray*}
\hat{h}_{a\rightarrow i}^{(k)(l)}[j] = \frac{1}{{\sigma}^{(k)}[j]}({h}_{a\rightarrow i}^{(k)(l)}[j]-{\mu}^{(k)}[j]),
\end{eqnarray*}
where ${h}_{a\rightarrow i}^{(k)(l)}[j]$ is the $j$th element of the hidden features while calculating the damping ratio ${\alpha}_{a\rightarrow i}^{(k)(l)}$, 
\begin{equation*}
    \mu^{(k)}[j] = \frac{1}{Z}\sum_{a,i,l}{h}_{a\rightarrow i}^{(k)(l)}[j], ~~\sigma^{(k)}[j] = \sqrt{\frac{1}{Z}\sum_{a,i,l}({h}_{a\rightarrow i}^{(k)(l)}[j]-\mu^{(k)}[j])^2+\epsilon}.
\end{equation*}

\paragraph{Baselines}
We implement the classical beam search algorithm as follows: During search, the algorithm maintains a cache which contains K states whose log-scores are the largest among all visited states. The cache is first initialized with a random state. At each step, the algorithm examines all neighbors of the states in the cache and updates the cache accordingly. We define a state $x$ is a neighbor of a state $x'$ if and only if their variable assignments are only different on one variable, i.e., $\exists i~x_i\neq x'_i\wedge~\forall j\neq i~ x_j\equiv x'_j$. The algorithm stops when there is no update of the cache, when the maximum search step is reached, or when the maximum search time is reached. We set the maximum search step as 100000, the maximum search time as one hour per instance, and the size of the cache $K$ as 10. The best-first search algorithm is implemented by setting the size of the cache $K$ as 1. 

% However, due to the complexity of the factor graphs in UAI-challenge datasets, the classical beam search algorithms are very slow. Therefore, we also implement one faster randomized beam search algorithm and denote it as random search. The random search algorithm is implemented similarly to the beam search algorithm. However, at each step, the algorithm only examines $M$ randomly selected neighbors of the states in the cache and updates the cache accordingly. We set $M$ as 100, the size of the cache $K$ as 100, and the maximum search step as $100000$ in our experiments. 
%\end{lemma}

%Optionally include extra information (complete proofs, additional experiments and plots) in the appendix. This section will often be part of the supplemental material.

\end{document}